\documentclass[sigconf,nonacm]{aamas}

\usepackage{booktabs}
\usepackage{amsmath}
\usepackage{multirow}
\usepackage{subcaption}
\usepackage{pifont}
\usepackage{listings}
\usepackage{xcolor}
\usepackage[vlined,ruled,linesnumbered]{algorithm2e}
\newcommand{\cmark}{\ding{51}}
\newcommand{\xmark}{\ding{55}}
\DeclareMathOperator*{\argmax}{argmax}
\DeclareMathOperator*{\argmin}{argmin}
\SetKwBlock{DataCollection}{\normalfont\textsc{Data Collection}:}{}
\SetKwBlock{PolicyUpdate}{\normalfont\textsc{Policy Update}:}{}
\SetKwBlock{WorstCaseCritic}{\normalfont\textsc{Worst-Case Critic Update}:}{}
\SetKwBlock{CriticUpdate}{\normalfont\textsc{Critic Update}:}{}
\SetKwBlock{TrainAdversary}{\normalfont\textsc{Train the adversary}:}{}
\setcopyright{none}
\title{Robust Policy Optimization via Adversarial Importance Sampling}

\author{Amine Andam}
\affiliation{
  \institution{Mohammed VI Polytechnic University}
  \country{Morocco}}
\email{andamamine83@gmail.com}

\author{Jamal Bentahar}
\affiliation{
  \institution{Khalifa University}
  \institution{Concordia University}
  \country{UAE}}
\email{jamal.bentahar@ku.ac.ae}
\author{Mustapha Hedabou}
\affiliation{
  \institution{Mohammed VI Polytechnic University}
  \country{Morocco}}
\email{mustapha.hedabou@um6p.ma}

\begin{abstract}
  Significant progress has been made in safeguarding Deep Reinforcement Learning (DRL) policies against input perturbations. Developing robust DRL involves three main stages: algorithm design, implementation, and evaluation. In this work, we identify and address a key limitation at each stage. First, we introduce Adversarial Importance Sampling (Advis), a method that estimates and optimizes worst-case returns directly from trajectories already collected during standard training. Advis satisfies three desirable criteria that prior work does not jointly achieve: it requires no additional environment interactions, no auxiliary networks, and captures long-term robustness. Second, we introduce advrl, a modular PyTorch library that provides clean, single-file implementations of existing robustness methods and adversarial attacks, enabling rapid prototyping and reproducible evaluations. Third, we revisit evaluation under learned adversaries and show that optimal adversarial hyperparameters do not transfer across agents, which can lead to an overestimation of robustness when using a limited set of attacker configurations. Accordingly, we evaluate policies against a large and diverse set of attackers, using 6 to 14 times more configurations than prior work. Finally, experiments on continuous control environments show that Advis achieves stronger worst-case performance and greater training efficiency than existing baselines.
\end{abstract}


\newcommand{\BibTeX}{\rm B\kern-.05em{\sc i\kern-.025em b}\kern-.08em\TeX}

\begin{document}

%%% The following commands remove the headers in your paper. For final 
%%% papers, these will be inserted during the pagination process.

%\pagestyle{fancy}
%\fancyhead{}

%%% The next command prints the information defined in the preamble.

\maketitle
\pagestyle{plain}
\thispagestyle{plain}

%%%%%%%%%%%%%%%%%%%%%%%%%%%%%%%%%%%%%%%%%%%%%%%%%%%%%%%%%%%%%%%%%%%%%%%%

\section{Introduction}
Deep Reinforcement Learning (DRL) agents are vulnerable to adversarially crafted input perturbations that cause them to take suboptimal or even catastrophic decisions \cite{DBLP:conf/ijcai/LinHLSLS17,DBLP:conf/iclr/HuangPGDA17}. To robustify a DRL algorithm, one either modifies how samples are collected by adding perturbed samples \cite{DBLP:journals/corr/abs-1712-03632,DBLP:conf/iclr/ZhangCBH21,DBLP:conf/iclr/SunZLH22}, or modifies the learning objective by adding a regularization term designed to promote robustness \cite{DBLP:conf/nips/0001CX0LBH20,DBLP:conf/nips/OikarinenZMDW21,DBLP:conf/nips/LiangSZH22}. Most existing methods suffer from at least one of the following issues (see Table \ref{tab:compare-work}): (1) they require collecting additional samples from the environment, usually to train an adversary using DRL; (2) they require training additional networks to estimate the worst-case performance or to learn an adversarial policy; (3) they use state-wise robustness objectives rather than explicitly account for the cumulative effects of perturbations along a trajectory. In this paper, we propose \emph{Adversarial Importance Sampling (Advis)}, an algorithm that avoids these three limitations. 

Several existing methods use convex relaxation of neural networks \cite{DBLP:conf/nips/XuS0WCHKLH20} as an intermediate step in their algorithms to compute lower and upper bounds on policy outputs under input perturbations. For example, WocaR-RL~\cite{DBLP:conf/nips/LiangSZH22} uses these bounds to train an auxiliary network that estimates worst-case performance. Our key insight is that this additional learning can be avoided: the output bounds and trajectories already collected during standard training provide the information needed to estimate adversarial performance directly. Specifically, we use the bounds to approximate worst-case action probabilities. We then use importance sampling \cite{owen2013monte,DBLP:books/lib/SuttonB2018} to estimate adversarial returns from trajectories already collected during standard training. This way, we can obtain unbiased (or consistent) estimators of the worst-case performance without the need to train another network or collect additional samples. Moreover, the importance weights accumulate across successive actions, allowing the objective to account for perturbations over multiple steps. Based on these estimates, we propose different regularization schemes to train robust policies. We also present practical techniques to mitigate the high variance of importance sampling under the distribution mismatch between the nominal and adversarial policies.

In addition to our algorithmic contribution, we also examine the current state of adversarial DRL implementations. Existing codebases are often complex and not well suited for prototyping. New methods are typically implemented on top of existing codebases, which themselves are built on earlier implementations. This has resulted in monolithic codebases where a single file contains entangled implementations of multiple robustness methods, while also handling both training and evaluation. Moreover, no unified library implements most standard baselines and adversarial attacks, forcing researchers to combine multiple codebases for benchmarking. To address these issues, we introduce \texttt{advrl}, a modular PyTorch library that provides clean single-file implementations of existing methods together with tools for large-scale evaluation. We design the library to support rapid prototyping and reproducible evaluations.

Finally, we revisit how existing work benchmarks robustness against learned adversaries (i.e., adversaries that train neural networks or policies to generate perturbations). Their effectiveness depends heavily on their training hyperparameters. Consequently, adversarial evaluations strongly depend on how these hyperparameters are tuned. Importantly, we show that the hyperparameter configurations that produce the strongest adversaries against one algorithm are often ineffective against another algorithm. As a result, evaluating robustness against only a small set of adversaries can lead to misleading conclusions: an agent may appear robust due to insufficiently tuned hyperparameters. To address this, we evaluate against a large and diverse population of adversarial configurations. We considered 2328 learned attackers per agent compared to the 30 to 200 typically used in prior work. We focus on robustifying policy optimization algorithms and run extensive experiments on MuJoCo environments. We demonstrate consistent improvements over existing methods.

We make three key contributions, covering algorithm design, implementation, and adversarial evaluation:
\begin{itemize}
    \item We introduce \emph{Advis}, a new robust policy optimization algorithm that uses existing trajectories to estimate adversarial returns and train robust policies, without additional environment interactions or auxiliary networks.
    \item We introduce \texttt{advrl}, which is, to the best of our knowledge, the first unified library providing modular implementations of state-of-the-art robustness methods and learned adversarial attacks. It supports rapid prototyping and enables reproducible and large-scale evaluations.
    \item We conduct extensive evaluations against adaptive adversaries and provide insights into how they should be evaluated. Under this evaluation, Advis achieves the best robustness.
\end{itemize}

\section{Background and Related Work}
\begin{table}[t]
  \centering
  \caption{Comparison of robust DRL methods.}
  \begin{tabular}{lccc}
  \toprule
  Method & \# Samples & \# Networks & Long-term \\
  \midrule
  SA-ATLA     & $\uparrow$ & $\uparrow$ & \cmark \\
  PA-ATLA        & $\uparrow$ & $\uparrow$ & \cmark \\
  SA-Reg  & $=$ & $=$ & \xmark \\
  Radial    & $=$ & $=$ & \xmark \\
  WocaR-RL       & $=$ & $\uparrow$ & \cmark \\
  ACoE       & $=$ & $\uparrow$ & \cmark \\
  \midrule
  \textbf{Ours} & $=$ & $=$ & \cmark \\
  \bottomrule
  \end{tabular}
  \label{tab:compare-work}
\end{table}

\subsection{Preliminaries} 

\subsubsection{Deep Reinforcement Learning}
A Markov Decision Process (MDP) \cite{DBLP:books/wi/Puterman94} is defined by the tuple $(\mathcal{S}, \mathcal{A}, P, R, \rho_0, \gamma)$, where $\mathcal{S}$ and $\mathcal{A}$ denote the state and action spaces, $P: \mathcal{S} \times \mathcal{A} \times \mathcal{S} \to \Delta(\mathcal{S})$ is the transition probability distribution, $R: \mathcal{S} \times \mathcal{A} \to \mathbb{R}$ is the reward function, $\rho_0$ is the initial state distribution, and $\gamma \in [0,1)$ is the discount factor. A policy $\pi: \mathcal{S} \rightarrow \Delta(\mathcal{A})$ specifies a distribution over actions conditioned on the current state. We write $\tau \sim \pi$ to denote that the trajectory $\tau$ is generated by following policy $\pi$. We denote by $p(\tau| \pi)$ the trajectory distribution induced by $\pi$:
\begin{equation}
\label{eq:traj_dist}
  p(\tau| \pi) = \rho_0(s_0)\prod_{t=0}^{T-1} \pi(a_t | s_t)P(s_{t+1} | s_t, a_t)
\end{equation}
where $T$ is the length of the trajectory $\tau$.

The goal of DRL is to find a policy $\pi$ that maximizes the expected discounted return:
\[
  J(\pi)=\mathbb{E}_{\tau\sim\pi}\!\left[R(\tau)\right]=\mathbb{E}_{\tau\sim\pi}\!\left[\sum_{t=0}^{\infty}\gamma^t R(s_t,a_t)\right]
\]

Policy gradient methods achieve this by directly optimizing a parameterized policy $\pi_\theta$ via gradient ascent on $J(\pi_\theta)$. Proximal Policy Optimization (PPO) \cite{DBLP:journals/corr/SchulmanWDRK17} is a widely used policy optimization method with state-of-the-art performance. It minimizes the following loss:
\begin{equation}
\mathcal{L}^{\text{PPO}} =
\mathbb{E}_{t} \left[-
\min \left(
r_t A_t,\;
\mathrm{clip}(r_t, 1 \pm \epsilon_{\mathrm{clip}}) A_t
\right)
\right]
\end{equation}
where $r_t =\frac{\pi_\theta(a_t \mid s_t)}{\pi_{\theta_{\text{old}}}(a_t \mid s_t)}$, $A_t$ is the estimated advantage at step $t$, and $\epsilon_{\mathrm{clip}}$ is the clipping threshold.

\subsubsection{Importance sampling (IS)} IS \cite{owen2013monte} is a technique commonly used in RL to estimate the performance of a target policy $\pi$ using trajectories generated by a different policy $b$. Given trajectories $\{ \tau_i \sim b \}_{i=1}^N$ collected under the behavior policy $b$, we can estimate the performance of $\pi$ as follows:
\begin{equation}
  \label{eq:ois}
  \hat{J}_{\mathrm{IS}}(\pi) = \frac{1}{N} \sum_{i=1}^N \rho_{T-1}(\tau_i)\, R(\tau_i)
\end{equation}
where $\rho_t(\tau) = \prod_{k=0}^{t} \frac{\pi(a_k \mid s_k)}{b(a_k \mid s_k)}$ is the importance weight.

\subsubsection{Adversarial DRL}
We consider standard adversarial attacks on the observations of DRL agents \cite{DBLP:conf/nips/0001CX0LBH20,DBLP:conf/iclr/SunZLH22,DBLP:conf/nips/LiangSZH22}, in which an adversary perturbs the agent’s state observations within a bounded region during test-time. Instead of receiving the true state of the environment $s$, the agent receives a perturbed state $\tilde{s}$ and selects actions according to $\pi(.|\tilde{s})$ rather than $\pi(.|s)$. We specify an attack budget $\epsilon$ and restrict the perturbations to be $\ell_\infty$-bounded: $\tilde{s} \in \mathcal{B}_{\epsilon}(s) = \{ x \in \mathcal{S} : \|x - s\|_{\infty} \leq \epsilon \}$.

\subsubsection{Convex relaxation of neural networks} Convex relaxation methods \cite{DBLP:conf/nips/XuS0WCHKLH20} enable the computation of provable bounds of neural network outputs. Given a network $f : \mathbb{R}^n \to \mathbb{R}^k$ and a perturbation set 
$\mathcal{B}_{\epsilon}(x_0)$ around the input $x_0$,  we can find element-wise bounds $\underline{f}, \overline{f} \in \mathbb{R}^k$ such that:
\begin{equation}
  \label{eq:convex-relax}
  \underline{f}(x_0) \leq f(x) \leq \overline{f}(x_0) \quad \forall x \in \mathcal{B}_{\epsilon}(x_0)
\end{equation}
Well-known methods include Interval Bound Propagation (IBP) \cite{DBLP:journals/corr/abs-1810-12715} and CROWN \cite{DBLP:conf/nips/ZhangWCHD18}.

\subsection{Related Work}
Methods for robust DRL under state perturbations can be broadly divided into two families: regularization-based and attack-driven methods. Regularization-based methods augment the standard DRL objective with a robustness term ($\mathcal{L}_{\mathrm{DRL}} + \kappa \mathcal{L}_{\mathrm{robust}}$). These methods differ mainly in the design of $\mathcal{L}_{\mathrm{robust}}$. Typical choices include losses that optimize worst-case performance, or penalize discrepancies between policy distributions on clean and perturbed inputs.

Attack-driven methods instead modify how training data is collected: the agent is trained on adversarial trajectories generated under a state adversary $v: \mathcal{S} \rightarrow \mathcal{S}$:
\[
  \tilde{\tau}=(\dots , \tilde{s_t} =v(s_t), a_t \sim \pi(.| \tilde{s_t}),R(s_t,a_t), \dots)
\]
The main focus of this line of work is to design optimal adversaries $v$. Below, we review both categories in detail.

\paragraph{SA-Reg \cite{DBLP:conf/icml/ShenLJWZ20,DBLP:conf/nips/0001CX0LBH20}:} 
The core idea of SA-Reg is to minimize the discrepancy between $\pi(.|s_t)$ and $\pi(.|\tilde{s_t})$ using a divergence measure $\mathcal{D}$: 
\begin{equation}
  \label{eq:sa-reg-loss}
  \mathcal{L}^{\mathrm{SA-Reg}} = \max_{\tilde{s_t} \in \mathcal{B}_{\epsilon}(s_t)} \mathcal{D}(\pi(.|s_t), \pi(.|\tilde{s_t}))
\end{equation}

\paragraph{Radial \cite{DBLP:conf/nips/OikarinenZMDW21}:} Radial minimizes an upper bound $\mathcal{L}^{\mathrm{Radial}}$ on the standard DRL objective under adversarial perturbations:  \[\mathcal{L}^{\text{PPO}}(\tilde{s}) \leq \mathcal{L}^\mathrm{Radial}(s),  \quad \forall \tilde{s} \in \mathcal{B}_\epsilon(s) \] 

To obtain this bound for policy gradient methods, Radial replaces the current policy $\pi_\theta$ in the DRL loss with a \emph{worst-case policy} $\pi^{\mathrm{wc}}_\theta$. For PPO, they suggest the following:
\begin{equation}
  \mathcal{L}^\mathrm{Radial} = \mathbb{E}_{t} \left[ -
  \min \left(
  r_t^{{\mathrm{wc}}} A_t,
  \mathrm{clip}\big(r_t^{{\mathrm{wc}}}, 1 \pm \epsilon_{\mathrm{clip}}\big)A_t
  \right)
  \right]
\end{equation}
where $r_t^{{\mathrm{wc}}}=\frac{\pi^{\mathrm{wc}}_\theta(a_t \mid s_t)}{\pi_{\theta_{\text{old}}}(a_t \mid s_t)}$.

Radial constructs $\pi^{\mathrm{wc}}$ by assigning minimal probability to favorable actions ($A_t > 0$) and maximal probability to unfavorable ones ($A_t \leq 0$). For a Gaussian policy $\pi_\theta(.| s) = \mathcal{N}(\mu_{\theta}(s), \Sigma_{\theta})$, and a perturbation set $\mathcal{B}_\epsilon(s)$,  $\pi^{\mathrm{wc}}(a_t | s_t)$ admits the following closed form:
\begin{equation}
\label{radial-wc-policy}
\pi^{\mathrm{wc}}_{\theta}(a_t \mid s_t) =
  \begin{cases} 
  \mathcal{N}(a_t; \mu^{+}_{\theta}(s_t), \Sigma_{\theta}), & \text{if } A_t > 0 \\
  \mathcal{N}(a_t; \mu^{-}_{\theta}(s_t), \Sigma_{\theta}), & \text{otherwise}
  \end{cases}
\end{equation}
given that:
\[
\mu^{+}_{\theta}(s_t) =
\arg\max_{\mu \in
[\underline{\mu}_{\theta}(s_t);\overline{\mu}_{\theta}(s_t)]}
\|a_t-\mu\|^2_{\Sigma_{\theta}}
\]
\[
\mu^{-}_{\theta}(s_t) =
\arg\min_{\mu \in
[\underline{\mu}_{\theta}(s_t);\overline{\mu}_{\theta}(s_t)]}
\|a_t-\mu\|^2_{\Sigma_{\theta}}
\] 
where $\underline{\mu}_{\theta}$, $\overline{\mu}_{\theta}$ denote the lower and upper bounds computed using convex relaxation tools and $\| x \|^2_{\Sigma} = x^\top \Sigma^{-1} x$.

\paragraph{SA-ATLA~\cite{DBLP:conf/iclr/ZhangCBH21}:} SA-ATLA is an attack-driven method that learns optimal adversaries using DRL. It constructs an adversarial MDP $\overline{M} = (\mathcal{S}, \overline{\mathcal{A}}_{\text{adv}},\overline{P}, \overline{R},\gamma)$ whose optimal policy $v$ generates perturbations that minimize the return of a fixed victim policy $\pi$. In SA-ATLA, the adversary’s action space matches the state space: $ \overline{\mathcal{A}}_{\text{adv}} = \mathcal{S}$. At each step, the adversary observes $s_t$ and outputs a perturbed state $\tilde{s}_t =v(s_t)$. The agent then selects an action $a_t \sim \pi(.|\tilde{s}_t)$ and receives reward $R(s_t,a_t)$, while the adversary receives $\overline{r}_t = -R(s_t,a_t)$. The training alternates between optimizing $v$ against a fixed $\pi$ and updating $\pi$ using trajectories generated under the learned adversary. Although SA-ATLA has theoretical guarantees for finding optimal adversaries, it has two main limitations: the adversary's action space inherits the dimensionality of $\mathcal{S}$, which makes it sample-inefficient in high-dimensional environments, and training a second RL agent roughly doubles the computational cost.

\paragraph{PA-ATLA~\cite{DBLP:conf/iclr/SunZLH22}:} PA-ATLA uses the same alternating training method as SA-ATLA, but uses a more efficient adversary. It constructs an adversarial MDP $\overline{M} = (\mathcal{S}, \overline{\mathcal{A}}_{\text{adv}}, \overline{P}, \overline{R}, \gamma)$ where the adversarial action space matches the agent’s action space $\overline{\mathcal{A}}_{\text{adv}} = \mathcal{A}$. Rather than learning high-dimensional state perturbations, the adversary outputs a worst-case action at each state $v(s_t) = \hat{a}_t$. Subsequently, a targeted FGSM attack is performed to find the state perturbation that pushes the agent to select the adversarial action $\hat{a}_t$. For a deterministic policy $\pi$, this corresponds to solving:
\begin{equation}
  \tilde{s}_t = \argmin_{\overline{s} \in \mathcal{B}_{\epsilon}(s_t)} \|\pi(\overline{s}) - \hat{a}_t\|^2
\end{equation}

\paragraph{WocaR-RL~\cite{DBLP:conf/nips/LiangSZH22}:} WocaR-RL is a regularization-based approach that evaluates and optimizes the worst-case performance of a policy under adversarial perturbations. This is achieved by training a worst-case action-value network $\underline{Q}_{\phi}^{\pi}$ using the following targets: 
\begin{equation}
  y_t = R(s_t,a_t) + \gamma \min_{\hat{a} \in \mathcal{A}_{\mathrm{adv}}(s_{t+1}, \pi)} \underline{Q}^{\pi}_{\phi^{-}}(s_{t+1}, \hat{a})
\end{equation}
and $\mathcal{A}_{\mathrm{adv}}(s, \pi) = \{a \in \mathcal{A} : \exists\, \tilde{s} \in \mathcal{B}_{\epsilon}(s) \text{ s.t. } a \sim \pi(\cdot \mid \tilde{s})\}$ denotes the set of actions an adversary can induce via state perturbations. It can be computed efficiently using convex relaxation. The minimization $\min_{\hat{a} \in \mathcal{A}_{\mathrm{adv}}}$ is solved using 50-step projected gradient descent.

The policy $\pi$ is then optimized to favor actions with higher worst-case values by minimizing the following worst-case policy loss:
\begin{equation}
  \mathcal{L}^{\text{WocaR-RL}}_1 = -\mathbb{E}_{t}\Big[ \min\Big( r_t\underline{Q}_t, \mathrm{clip}\!\left(r_t,\,1\pm \epsilon_{\mathrm{clip}}\right) \underline{Q}_t\Big)\Big]
\end{equation}
where $\underline{Q}_t =
\underline{Q}_{\phi}^{\pi}(s_t,a_t)$.

In addition, WocaR-RL uses a weighted SA-Reg loss to prioritize \textit{critical states}, where suboptimal actions can have catastrophic consequences. The state importance weight is defined as:
\begin{equation}
  \label{eq:state-importance-weight}
  w(s) = V^{\pi}(s) - \min_{a \in \mathcal{A}}\, \underline{Q}^{\pi}_{\phi}(s, a),
\end{equation}
where the inner minimization is solved via gradient descent and the weighted SA-Reg loss is computed as follows:
\begin{equation}
    \mathcal{L}^{\text{WocaR-RL}}_{2} =   \mathbb{E}_{t} \left[ w(s_t) \max_{\tilde{s}_t \in \mathcal{B}_{\epsilon}(s_t)} \mathcal{D}(\pi_{\theta}(.|s_t), \pi_{\theta}(.|\tilde{s}_t))\right]
\end{equation}
The final WocaR-RL regularization loss is:
\begin{equation}
    \mathcal{L}^{\text{WocaR-RL}} = \kappa_1 \mathcal{L}^{\text{WocaR-RL}}_{1} + \kappa_2 \mathcal{L}^{\text{WocaR-RL}}_{2}
\end{equation}
Other representative approaches include Adversarial Counterfactual Error (ACoE) \cite{DBLP:conf/iclr/BelaireSV25}, which minimizes the difference between $\pi$'s returns under adversarial and normal settings. However, it requires training two additional networks: one to estimate a counterfactual error and another to learn a reward model over neighboring states (i.e., $R(\tilde{s},a), \forall\, \tilde{s} \in \mathcal{B}_{\epsilon}(s)$). We exclude ACoE~\cite{DBLP:conf/iclr/BelaireSV25} from our results, as it did not converge after 2 million training steps (see Section \ref{sec:experiments}). Other work explores related directions, such as analyzing the effect of adversarial training on policies~\cite{DBLP:conf/icml/Korkmaz24,DBLP:conf/aaai/Korkmaz23}, designing robust neural architectures~\cite{DBLP:conf/aaai/NieJF024}, and studying action-space robustness~\cite{DBLP:conf/aaai/LeeGTHS20,DBLP:conf/nips/LiuL21a}.

Synthesizing prior work suggests that an ideal robustness method should (1) require no additional samples, (2) avoid auxiliary networks, and (3) account for long-term robustness. Table~\ref{tab:compare-work} summarizes existing methods and shows that none satisfy all three criteria: SA-Reg and Radial do not account for worst-case returns, WocaR-RL estimates worst-case returns using an auxiliary network, and SA-ATLA and PA-ATLA require training an additional agent. We address this gap with a regularization approach that requires no extra samples or networks while explicitly optimizing long-term robustness.

\section{Importance Sampling for Robustness}

In this section, we present \textit{Adversarial Importance Sampling (Advis)}. We describe how to design importance sampling based robust objectives, and address two key challenges associated with IS: high variance and distribution mismatch.

\subsection{Robust Training Objectives}
\label{sec:adv_loss}

The core idea of our approach is to estimate the worst-case performance directly from information already available during the training of many robust DRL algorithms. Several existing methods use convex relaxation to compute lower and upper bounds on policy outputs under input perturbations. Our key insight is that, once these bounds have been computed, the additional learning used to estimate worst-case performance can be avoided. Specifically, given the current policy $\pi$ and attack budget $\epsilon$, we first use the output bounds to approximate worst-case action probabilities $\pi^{\mathrm{wc}}(a_t|s_t)$ using the closed form in Eq.~(\ref{radial-wc-policy}). Then, importance sampling offers us a straightforward approach to estimate worst-case performance from trajectories already collected during standard training, without training an additional estimator or collecting trajectories under a worst-case policy. The corresponding importance weights are:
\begin{equation}
  \rho^{\text{advis}}_{t} (\tau) = \prod_{k=0}^{t} \frac{\pi^{\text{wc}}_{\theta}(a_k|s_k)}{\pi_{\theta_{\text{old}}}(a_k|s_k)}
\end{equation}
Using these weights, we can estimate the returns of the worst-case policy:
\begin{equation}
  \label{eq:ois-advis}
  \hat{J}(\pi^{\mathrm{wc}}) = \frac{1}{N} \sum_{i=1}^N \rho^{\mathrm{advis}}_{T-1}(\tau_i)\, R(\tau_i)
\end{equation}

Our approach is sample-efficient, as it does not require collecting additional samples from the environment. The same collected samples are used twice: first to optimize nominal performance through the standard DRL loss, and second to train for robustness using the estimated adversarial returns. By estimating returns over trajectories, it naturally captures long-term robustness. We present three training objectives, each one encodes a different inductive bias for robustness.

\textbf{ (i) Maximize worst-case performance:} A direct approach is to maximize the estimated performance of the worst-case policy, encouraging high returns even under adversarial perturbations:
\begin{equation}
  \label{eq:wc_loss}
  \mathcal{L}^{\text{advis-wc}} = - \hat{J}(\pi^{\text{wc}})
\end{equation}

\textbf{ (ii) Minimize return divergence:} We minimize the Mean Squared Error (MSE) between clean and worst-case returns, encouraging similar performance in both settings: 
\begin{equation}
  \label{eq:mse_loss}
  \mathcal{L}^{\text{advis-mse}} = \left(J(\pi) - \hat{J}(\pi^{\text{wc}})\right)^2
\end{equation}

\textbf{ (iii) Minimize performance degradation:} Unlike the MSE loss, we only penalize drops in returns, ignoring cases where the worst-case return estimate exceeds the clean return, which can occur early in training due to exploration:
\begin{equation}
  \label{eq:relu_loss}
  \mathcal{L}^{\text{advis-relu}} = \max \left(0,J(\pi) - \hat{J}(\pi^{\text{wc}})\right)
\end{equation}

\subsection{Variance and Distribution Mismatch}
Importance sampling estimators may suffer from high variance, especially under distribution mismatch between behavior and target policies \cite{DBLP:journals/jmlr/MetelliPMR20}. In our setting, this mismatch is amplified since the target policy is the worst-case policy $\pi^{\text{wc}}$, which can significantly deviate from $\pi$. As a result, importance weights $\rho^{\text{advis}}$ can collapse to zero or explode to large values, leading to unstable training.

We address this in two ways. First, we reduce distribution mismatch using $\epsilon$-scheduling: we gradually increase the attack budget during training, limiting the deviation between $\pi$ and $\pi^{\text{wc}}$ in the early stages. As the policy becomes more robust over the course of training, it naturally becomes less sensitive to adversarial perturbations, and the distribution mismatch shrinks accordingly. Second, we use IS estimators that have lower variance than ordinary importance sampling (OIS) in Eq.~(\ref{eq:ois}), namely Per-Decision IS (PDIS), Weighted IS (WIS), and Weighted Per-Decision IS (WPDIS). PDIS is an unbiased estimator and can reduce variance relative to OIS \cite{DBLP:conf/icml/PrecupSS00}. WIS and WPDIS are biased but consistent \cite{owen2013monte,DBLP:conf/icml/PrecupSS00}, and exhibit significantly lower variance than OIS \cite{DBLP:journals/jmlr/MetelliPMR20}.

We further stabilize training by computing these estimators over sampled trajectory windows rather than full trajectories, mitigating the sensitivity of IS estimators to trajectory length (i.e., the Curse of Horizon \cite{DBLP:conf/nips/Voloshin00Y21,DBLP:conf/nips/LiuLTZ18}). Finally, to prevent numerical instability, we truncate importance weights for PDIS and apply the log-sum-exp trick for WIS. See Appendix B.6 for more details.

\section{Revisiting Implementations and Evaluations}
\subsection{Codebase Landscape}

\begin{figure}[!htbp]
  \centering
  \includegraphics[width=\columnwidth]{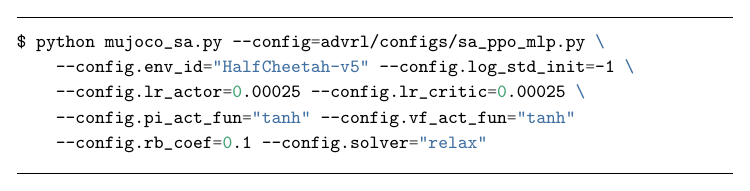}
  \caption{Command for training SA-Reg with \texttt{advrl}.}
  \Description{Example command for training SA-Reg with \texttt{advrl}.}
  \label{fig:advrl-training}
\end{figure}

\begin{figure}[!htbp]
  \centering
  \includegraphics[width=\columnwidth]{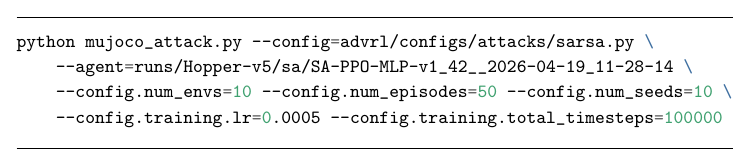}
  \caption{Command for evaluating against Robust SARSA.}
  \Description{Example command for evaluating an agent against a Robust Sarsa (RS) attack.}
  \label{fig:advrl-attack}
\end{figure}

\begin{figure*}[t]
  \centering
  \includegraphics[width=0.65\textwidth]{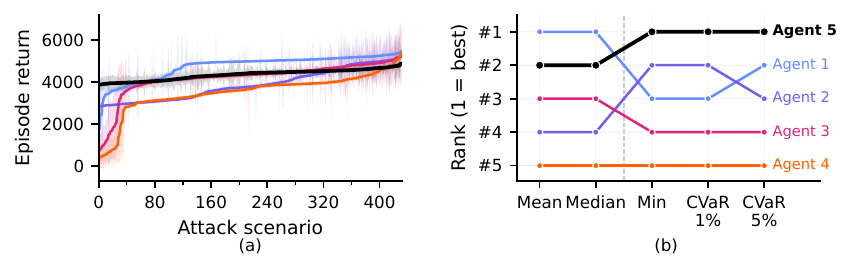}
  \caption{(a) Return distributions of five agents across 432 PA-AD attacks, sorted in ascending order. (b) Ranking of agents by different metrics. }
  \Description{(a) Return distributions of five agents across 432 PA-AD attacks, sorted in ascending order. (b) Ranking of agents by different metrics. }
  \label{fig:eval_metrics}
\end{figure*}

\begin{figure}[thbp]
  \centering
  \includegraphics[width=0.35\textwidth]{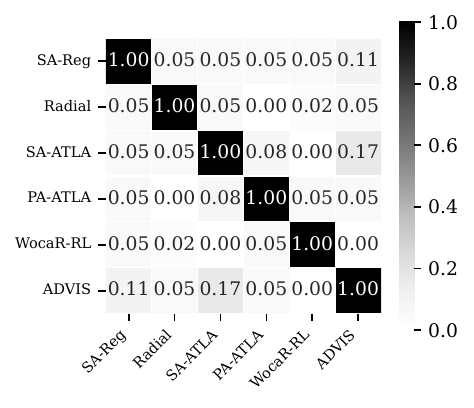}
  \caption{Jaccard similarity of top-5\% strongest PA-AD configurations across six defenses.}
  \Description{Jaccard similarity of top-5\% strongest PA-AD configurations across six defenses.}
  \label{fig:vulnerabilities}
\end{figure}

Despite recent progress in robust DRL, existing implementations remain difficult to use for prototyping and extensive evaluations. First, current codebases are often recursive and monolithic: newer methods build on top of earlier ones (e.g., WocaR-RL builds on SA-ATLA, itself built on SA-Reg), and a single file (\texttt{agent.py}) implements multiple algorithms together with the adversarial evaluation. This makes it difficult to trace the behavior of individual algorithms and complicates prototyping. Moreover, the recursive design can propagate errors across codebases. For example, an incorrect implementation of SGLD in SA-Reg is inherited by all the other codebases. Second, inconsistencies exist between published pseudocode and implementations: some components are not implemented (e.g., state importance weights, Eq. \ref{eq:state-importance-weight} in WocaR-RL), or different training objectives are used. Third, existing implementations rely on legacy MuJoCo versions that are no longer maintained (-v2). Finally, how checkpoints are saved limits portability. They can only be loaded with the current codebases and are tied to legacy MuJoCo.

To address these challenges, we introduce \texttt{advrl}\footnote{\url{https://github.com/AmineAndam04/advrl}}, an open-source PyTorch library that provides re-implementations of existing robustness algorithms and adversarial attacks. The library offers: (1) clean and modular single-file implementations, (2) reproducibility via standardized configurations, controlled seeds, and portable checkpoints, (3) parallel evaluations across CPU cores, (4) traceability through detailed logging during evaluations, and (5) easy integration of new methods. We also revisit several implementation details from prior work. Additional details are in Appendices A and B. Figures~\ref{fig:advrl-training} and~\ref{fig:advrl-attack} illustrate how to run training and adversarial evaluation with \texttt{advrl}.

\subsection{How to Evaluate Learned Adversaries}

\label{sec:how_to_evaluate}
We evaluate the robustness of agents by testing them against a suite of adversarial attacks, including: (1) Robust SARSA (RS) \cite{DBLP:conf/nips/0001CX0LBH20}, (2) SA-RL \cite{DBLP:conf/iclr/ZhangCBH21}, and (3) PA-AD \cite{DBLP:conf/iclr/SunZLH22}. These three are \textit{learned} adversaries. They are mainly used to evaluate the long-term robustness of an agent. They require training either a $Q$-network (RS) or a policy (SA-RL, PA-AD), which makes the evaluation sensitive to the choice of hyperparameters.

We identify two key requirements for fair and meaningful evaluation against learned adversaries. First, evaluation must consider a large and diverse set of adversaries (i.e., hyperparameter configurations), as configurations effective against one defense often fail against others. To demonstrate this, we train six defenses and evaluate them against 432 PA-AD attackers obtained via grid search (1 hyperparameter configuration = 1 attacker). For each defense, we select the top 5\% strongest configurations (lowest return) and measure their overlap using Jaccard similarity. Results are shown in Figure~\ref{fig:vulnerabilities}. The low similarity shows that optimal adversarial configurations do not transfer across defenses. This highlights the need to tune adversaries against each defense. Prior work evaluates against relatively few adversaries (e.g., 30, 50, or 216) and often does not report the tuned hyperparameters or the search space. In contrast, we evaluate each agent against a large set of adversaries: 1296 PA-AD attackers (432 configurations $\times$ 3 seeds), 600 SA-RL attackers, and 432 RS attackers. We report the search space in Appendix C.2.

The second key aspect is the evaluation metric. Mean or median returns across hyperparameter configurations can be misleading: an agent that performs well on average may still suffer catastrophic failures. This is illustrated in Figure~\ref{fig:eval_metrics} (left), which shows the returns of five agents across 432 PA-AD attacks, sorted in ascending order. Agent 5 (black) is the most robust because of the absence of catastrophic returns (near x=0). Yet, as shown in Figure~\ref{fig:eval_metrics} (right), mean and median fail to rank it first, and instead select an agent with catastrophic failures (Agent 1). This motivates the use of tail-sensitive metrics such as the minimal return and Conditional Value-at-Risk $\text{CVaR}_{\alpha}$~\cite{DBLP:conf/nips/ChowG14,DBLP:conf/nips/ChowTMP15} (the average return over the worst $100\alpha$\% of attackers). $\text{CVaR}_{\alpha}$ allows us to assess robustness at different risk levels. The choice of $\alpha$ should depend on the application: in safety-critical settings, where rare failures are unacceptable, we should use small $\alpha$ to focus on the strongest attacks. In less critical settings, evaluating agents only by their worst-case return can be overly conservative. Larger values of $\alpha$ are more appropriate when we also care about performance against a broader set of attacks. Different values of $\alpha$ can therefore lead to different rankings, depending on how much importance we give to the most severe failures.

\section{Experiments}
\label{sec:experiments}

In this section, we evaluate the robustness and computational cost of Advis compared to existing methods. We also study how the importance sampling estimator and adversarial loss
affect robustness. Finally, we compare grid search with random search to investigate whether strong attacks can be found using fewer configurations.

\begin{table*}[!htbp]
  \centering
  \small
  \setlength{\tabcolsep}{1mm}
  \caption{Minimal episodic reward  $\pm$  standard deviation.}
  \begin{tabular}{lcccccc}
  \toprule
  \multirow{2}{*}{Env.} & \multirow{2}{*}{Method} & \multirow{2}{*}{\shortstack{Natural\\ Reward}} & \multicolumn{3}{c}{Attack Reward} & \multirow{2}{*}{\shortstack{Best\\ Attack}} \\
  & & & RS & SA-RL & PA-AD & \\
  \midrule
  % ENV 1
  \multirow{6}{*}{\shortstack{HalfCheetah-v5\\ $\epsilon=0.15$}} & SA-Reg & 4903.3 $\pm$ 36 & 3795.7 $\pm$ 92 & 3940.0 $\pm$ 601 & 3796.9 $\pm$ 41 & 3795.7 \\
  & Radial & 4053.8 $\pm$ 522 & 4028.1 $\pm$ 533 & 4037.2 $\pm$ 472 & 4033.1 $\pm$ 536 & 4028.1 \\
  & SA-ATLA & 5400.8 $\pm$ 794 & 973.7 $\pm$ 131.6 & 342.6 $\pm$ 470 & 446.8 $\pm$ 607 & 342.6 \\
  & PA-ATLA & 4293.5 $\pm$ 735 & 2576.7 $\pm$ 1423 & 1933.8 $\pm$ 1098 & 2115.4 $\pm$ 1130 & 1933.8 \\
  & WocaR-RL \ & 5666.4 $\pm$ 412 & 401.4 $\pm$ 252 & -88.8 $\pm$162 & 50.85 $\pm$903 & -88.8 \\
  & Advis & 4781.1 $\pm$ 183 & 4118.8 $\pm$ 62 & 4282.4 $\pm$ 539 & 4111.2 $\pm$ 729 & \textbf{4111.2} \\
  \midrule
  % ENV 2 (example placeholder)
  \multirow{6}{*}{\shortstack{Hopper-v5\\ $\epsilon=0.075$}} & SA-Reg & 3555.2 $\pm$ 5 & 1363.32 $\pm$ 265 & 1128.5 $\pm$ 917 & 1140.5 $\pm$ 656 & 1128.5 \\
  & Radial & 3487.2 $\pm$ 76 & 858.6 $\pm$ 52 & 888.07 $\pm$ 118 & 937 $\pm$ 162 & 858.6 \\
  & SA-ATLA & 3261.34 $\pm$2 & 868.8 $\pm$ 52 & 1451.9 $\pm$ 112 & 896.3 $\pm$ 33 & 868.8 \\
  & PA-ATLA & 3105.9 $\pm$ 3 & 614.5 $\pm$ 71 & 772.3 $\pm$ 100 & 856.2 $\pm$ 593 & 614.5 \\
  & WocaR-RL \ & 3317.3 $\pm$ 1 & 595.9 $\pm$ 499 & 761.4$\pm$ 60 & 727.3 $\pm$ 150 & 595.9 \\
  & Advis & 2599.9 $\pm$ 413 & 1613.1 $\pm$ 307 & 1569.2 $\pm$ 267 & 1641.5 $\pm$ 235 & \textbf{1569.2} \\
  \midrule
  % ENV 3 (example placeholder)
  \multirow{6}{*}{\shortstack{Walker2d-v5\\ $\epsilon=0.05$}} & SA-Reg & 6129 $\pm$ 883 & 802.6 $\pm$ 1691 & 400.9 $\pm$ 2097 & 591 $\pm$ 2123 & 400.9 \\
  & Radial & 3114 $\pm$ 188 & 543.3 $\pm$ 658 &792.3 $\pm$ 408 & 824.4 $\pm$ 425 & 543.3 \\
  & SA-ATLA & 6132.4 $\pm$ 344 & 222.1 $\pm$ 1862 & 379 $\pm$ 2010 & 418 $\pm$ 2289 & 222.1 \\
  & PA-ATLA & 5781.4$\pm$ 127 & 1623.9 $\pm$ 1286 & 4701.1 $\pm$ 1779 & 3736.6 $\pm$ 1368 & 1623.9 \\
  & WocaR-RL \ & 6402.7 $\pm$ 294 & 184.9 $\pm$ 1907 & 257 $\pm$ 1603 & 284.3 $\pm$ 2294 & 184.9 \\
  & Advis & 4626.8 $\pm$ 907 & 3189.6 $\pm$ 1222.9 & 2155.3 $\pm$ 1463 & 3619.6 $\pm$ 1081 & \textbf{2155.3} \\
  \bottomrule
  \end{tabular}
  \label{tab:results}
\end{table*}

\begin{figure*}[htb]
  \centering
  \includegraphics[width=0.8\textwidth]{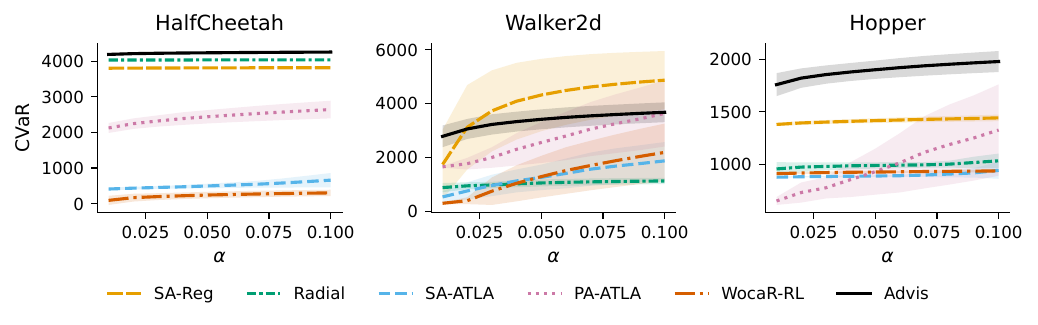}
  \caption{$\text{CVaR}_{\alpha}$ curves of the baselines and Advis in each environment. }
  \Description{$\text{CVaR}_{\alpha}$ curves of the baselines and Advis in each environment. }
  \label{fig:cvar_curves}
\end{figure*}

\subsection{Experimental Setup} 

We perform extensive experiments on three MuJoCo environments \cite{DBLP:journals/corr/abs-2407-17032}. For a fair comparison, all methods are trained for the same number of steps (2 million steps) using PPO. We use an identical network architecture: MLP of 2 layers with 64 neurons each. To isolate each method’s contribution, we evaluate them independently (e.g., no SA-Reg loss is added to other methods as in previous work). For each environment, we use the attack budget commonly used in prior work. We use a grid search to identify optimal hyperparameters. To select the best hyperparameters, we first discard agents with poor clean performance. We then evaluate the remaining agents against 20 randomly sampled learned adversaries and eliminate non-robust ones. Finally, we evaluate the retained agents against the full set of learned attacks and select the configuration with the best robust performance. All agents are trained with three random seeds. Each attack is evaluated over 1000 episodes (20 seeds, 50 episodes per seed). We report the interquartile mean (IQM) \cite{DBLP:conf/nips/AgarwalSCCB21} and standard deviation of episodic returns. We used Intel Xeon Platinum CPUs. Evaluating the six reported methods against learned adversaries requires approximately 8292 CPU-hours per environment, excluding the grid search for robustness training and the additional evaluations used for hyperparameter selection. We provide additional results with A3C~\cite{DBLP:conf/icml/MnihBMGLHSK16}, another environment, and different attack budgets in the appendix.
\subsection{Numerical Results}
\begin{figure*}[htb]
  \centering
  \begin{subfigure}[t]{0.8\textwidth}
      \centering
      \includegraphics[width=\textwidth]{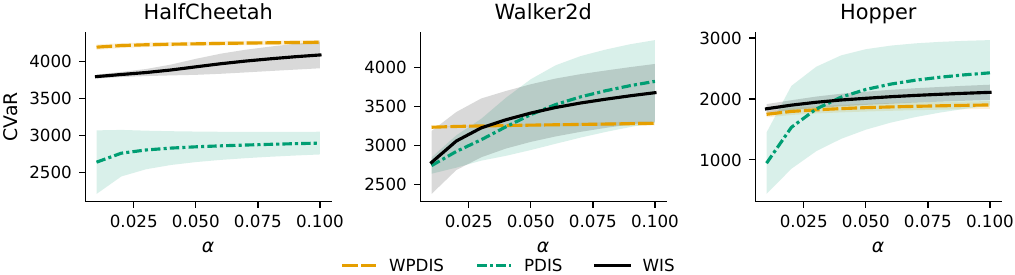}
      \caption{Ablation: importance sampling technique.}
      \Description{Ablation: importance sampling technique.}
      \label{fig:is_ablation}
  \end{subfigure}
  \vspace{0.5em}
  \begin{subfigure}[t]{0.8\textwidth}
      \centering
      \includegraphics[width=\textwidth]{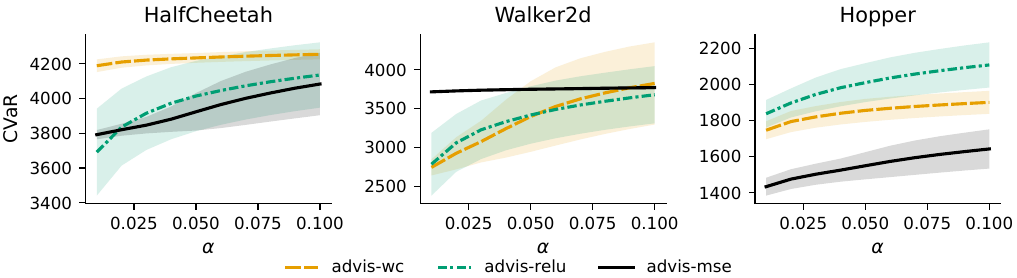}
      \caption{Ablation: adversarial loss.}
      \Description{Ablation: adversarial loss.}
      \label{fig:loss_ablation}
\end{subfigure}
  \caption{Ablation study results. (a) We ablate the importance sampling estimator. (b) We ablate the adversarial loss. The labels wc, MSE, and ReLU refer to the losses in Eqs. \ref{eq:wc_loss}, \ref{eq:mse_loss}, and \ref{eq:relu_loss}, respectively. }
  \Description{Ablation study results. (a) We ablate the importance sampling estimator. (b) We ablate the adversarial loss. The labels wc, MSE, and ReLU refer to the losses in Eqs. \ref{eq:wc_loss}, \ref{eq:mse_loss}, and \ref{eq:relu_loss}, respectively. }
  \label{fig:ablation_combined}
\end{figure*}

\begin{figure*}[htb]
  \centering
  \includegraphics[width=0.8\textwidth]{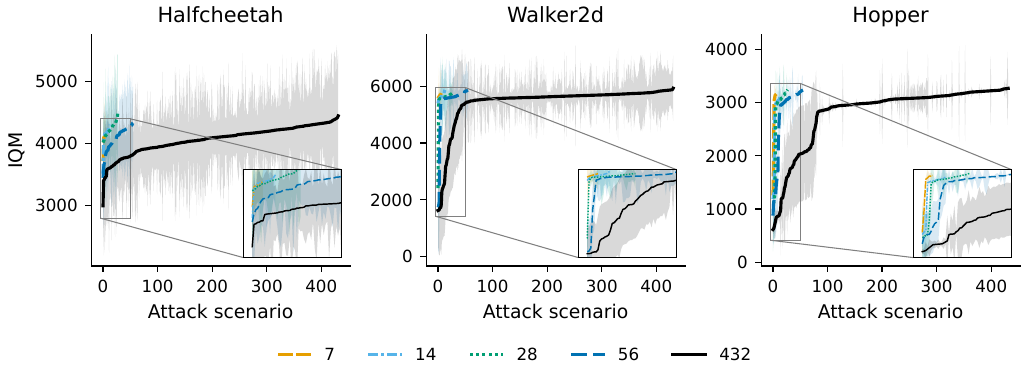}
  \caption{Grid search vs random search.}
  \Description{Grid search vs random search.}
  \label{fig:grid_ablation}
\end{figure*}
\begin{table*}[t]
  \centering
  \small
  \caption{Wall-clock training time (hours) on HalfCheetah.}
  \label{tab:runtime}
  \begin{tabular}{lccccc}
  \toprule
  Method & Total & ENV & EPOCH & Relax & Adv.\ loss \\
  \midrule
  PPO
  & $1.18\pm0.002$ & $0.25\pm0.002$ & $0.92\pm0.02$
  & -- & -- \\
  Advis
  & $0.87\pm0.03$ & $0.25\pm0.01$ & $0.58\pm0.02$
  & $0.15\pm0.008$ & $0.03\pm0.002$ \\
  Advis (closest)
  & $1.67\pm0.02$ & $0.41\pm0.008$ & $1.25\pm0.02$
  & $0.36\pm0.005$ & $0.07\pm0.0$ \\
  SA-Reg
  & $2.19\pm0.02$ & $0.38\pm0.0$ & $1.8\pm0.02$
  & $0.6\pm0.005$ & $0.05\pm0.0$ \\
  Radial
  & $2.4\pm0.01$ & $0.38\pm0.002$ & $2\pm0.02$
  & $0.67\pm0.002$ & $0.1\pm0.0$ \\
  PA-ATLA
  & $4.57\pm0.0$ & $2.83\pm0.02$ & $1.73\pm0.0$
  & -- & -- \\
  SA-ATLA
  & $2.53\pm0.04$ & $0.77\pm0.02$ & $1.73\pm0.02$
  & -- & -- \\
  WocaR-RL
  & $2.49\pm0.09$ & $0.38\pm0.01$ & $2.1\pm0.02$
  & $0.01\pm0$ & -- \\
  \bottomrule
  \end{tabular}
\end{table*}

\subsubsection{ Main results} 

Table~\ref{tab:results} reports the minimal performance achieved under each attack. The last column ("Best Attack") shows the worst-case return across all attacks. The results indicate that our method consistently achieves the strongest robust performance across all environments. In Figure~\ref{fig:cvar_curves}, we report $\text{CVaR}_\alpha$ curves for  $\alpha \in [0.01,0.1]$. In HalfCheetah and Hopper, Advis dominates all baselines across the entire range of $\alpha$.  Walker2d shows a more nuanced behavior: for small $\alpha \leq 0.02$ (extreme tail events), Advis outperforms SA-Reg, whereas for $\alpha > 0.02$, the opposite is true. This illustrates why the choice of the best method should account for the risk-sensitivity: a conservative evaluation favors Advis, while a more tolerant evaluation may favor SA-Reg, especially if it also offers lower computational cost.

\subsubsection{Training time}
Table~\ref{tab:runtime} reports the average wall-clock training time on HalfCheetah. Total is the full training time, ENV is the time spent collecting and processing environment data, and EPOCH is the total optimization time. Relax and Adv.\ loss report the time spent computing convex relaxations and the adversarial loss, respectively.

The differences in training time are partly due to the implementation: Advis uses five parallel environments, compared to two for the baselines, and therefore uses larger batches. We also report "Advis (closest)", which uses a batch size close to the baselines. Both Advis variants train faster than the other robustness methods.

\subsubsection{Ablations}

We ablate two critical components of our method: the importance sampling technique and the adversarial loss. Figure~\ref{fig:is_ablation} compares the $\text{CVaR}_\alpha$ curves for three IS estimators: PDIS, WIS, and WPDIS. In general, the weighted variants (WIS and WPDIS) outperform PDIS, particularly at small $\alpha$ values. Moreover, WPDIS exhibits significantly lower variance than the other techniques. In Figure~\ref{fig:loss_ablation}, we evaluate the impact of different adversarial losses introduced in Section~\ref{sec:adv_loss}. No single loss consistently outperforms the others across all environments. The best loss depends on the environment, highlighting the importance of tuning the adversarial loss.

\subsubsection{Grid search vs random search} From a red-teaming perspective, we are interested in running a full grid search over the hyperparameters of learned adversaries to ensure a thorough, fair, and reproducible evaluation. Unlike random search, the determinism of grid search avoids introducing additional randomness into the evaluation of RL agents. However, random search can be more efficient, as it may find competitive attacks with far fewer configurations \cite{DBLP:journals/jmlr/BergstraB12}. To assess this, we compare the two methods by training RS attacks against PA-ATLA agents (we select PA-ATLA to test whether random search can find stronger attacks than those reported in Table \ref{tab:results}). We use random search with budgets of 7, 14, 28, and 56 configurations, sampled from the same hyperparameter domain as the grid search.  In Figure~\ref{fig:grid_ablation}, we report the sorted attack performance for both methods. None of the random search trials find stronger attacks than grid search. As the number of sampled configurations increases, random search approaches the strongest attacks found by grid search.

\textbf{Limitations and future work:} We focused on policy optimization methods in continuous control environments. We made this choice to be able to run extensive experiments. In future work, we will extend our approach to other RL paradigms. For \texttt{advrl}, we will continue integrating newly proposed robustness methods as they appear. Moreover, we will add support for automated hyperparameter tuning for learned adversaries.

\section{Conclusion}
In this paper, we proposed Advis, a method for training robust DRL agents against input perturbations. We showed how Advis estimates and optimizes worst-case performance directly from trajectories already collected during standard training. In contrast to prior work, we do not require collecting additional environment trajectories or training auxiliary neural networks, while explicitly accounting for long-term robustness. We also introduced \texttt{advrl}, a unified and modular library that provides clean implementations of existing methods and facilitates experimentation. Moreover, we conducted evaluations using a large population of adversarial configurations for a reliable robustness assessment. Empirically, Advis achieves the highest worst-case returns in the tested benchmarks. Moreover, our timing results show that it trains faster than the other robustness methods.

%%%%%%%%%%%%%%%%%%%%%%%%%%%%%%%%%%%%%%%%%%%%%%%%%%%%%%%%%%%%%%%%%%%%%%%%

%%% The next two lines define, first, the bibliography style to be 
%%% applied, and, second, the bibliography file to be used.

\bibliographystyle{ACM-Reference-Format} 
\bibliography{sample}

%%%%%%%%%%%%%%%%%%%%%%%%%%%%%%%%%%%%%%%%%%%%%%%%%%%%%%%%%%%%%%%%%%%%%%%%

\clearpage
% Technical appendix included by main.tex.
\onecolumn
\begin{center}
  {\LARGE\bfseries Technical Appendix\par}
\end{center}

\appendix

\section{Introduction to \texttt{advrl} library }
\label{sec:suppl_advrl}

\texttt{advrl} is a library built using Pytorch that provides clean implementations of adversarial Deep Reinforcement Learning robustness and attack algorithms. It focuses on perturbations against inputs of the policy networks, a widely studied problem. Our goal is to provide a tool that facilitates research and investigation of adversarial DRL. To the best of our knowledge, \texttt{advrl} is the first project to give researcher a single library that has ready implementations of the main methods in one place. Moreover, our library provides modular and clean implementation, which facilitates prototyping. In addition, our library ensures reproducibility and traceability of results for responsible research. In the following, we will provide information on how to use our library details regarding the implemented baselines.  \texttt{advrl} will be maintained by the authors, and it will keep implementing new SOTA algorithm as they appear. 

\texttt{advrl} provides the implementations of the following algorithms:
\begin{itemize}
    \item Vanilla training: Proximal Policy Optimization (PPO) \cite{DBLP:journals/corr/SchulmanWDRK17}
    \item Defensive training: SA-Reg \cite{DBLP:conf/icml/ShenLJWZ20,DBLP:conf/nips/0001CX0LBH20}, Radial \cite{DBLP:conf/nips/OikarinenZMDW21}, SA-ATLA~\cite{DBLP:conf/iclr/ZhangCBH21}, PA-ATLA~\cite{DBLP:conf/iclr/SunZLH22}, WocaR-RL \cite{DBLP:conf/nips/LiangSZH22}, and ADVIS (ours).
    \item Adversarial attacks: Random perturbations (Uniform, Gaussian and fixed), Maximal Action Difference (MAD) attack \cite{DBLP:conf/nips/0001CX0LBH20}, Robust Sarsa (RS) attack \cite{DBLP:conf/nips/0001CX0LBH20}, Critic attack \cite{DBLP:conf/atal/PattanaikTLBC18,DBLP:conf/nips/0001CX0LBH20}, State-Adversarial attacker learned using RL (SA-RL)~\cite{DBLP:conf/iclr/ZhangCBH21}, Policy Adversarial Actor Director (PA-AD)~\cite{DBLP:conf/iclr/SunZLH22}.
\end{itemize}

\subsection{How to use \texttt{advrl}}

Each algorithm in \texttt{advrl} is associated with a dedicated configuration file created using \texttt{ml\_collections}. Each configuration file fully specifies the training setup, including the neural network architecture, environment, adversarial training parameters, optimization hyperparameters, and logging options. See Figure \ref{fig:config-file}, for an example of the config file of SA-Reg algorithm. All the hyperparameters can be overridden directly via command-line flags without modifying the config file. Figure~\ref{fig:training_command} shows an example of a command to training SA-Reg.

\begin{figure}[!htbp]
    \centering
    \includegraphics[width=0.85\textwidth,height=0.9\textheight,keepaspectratio]{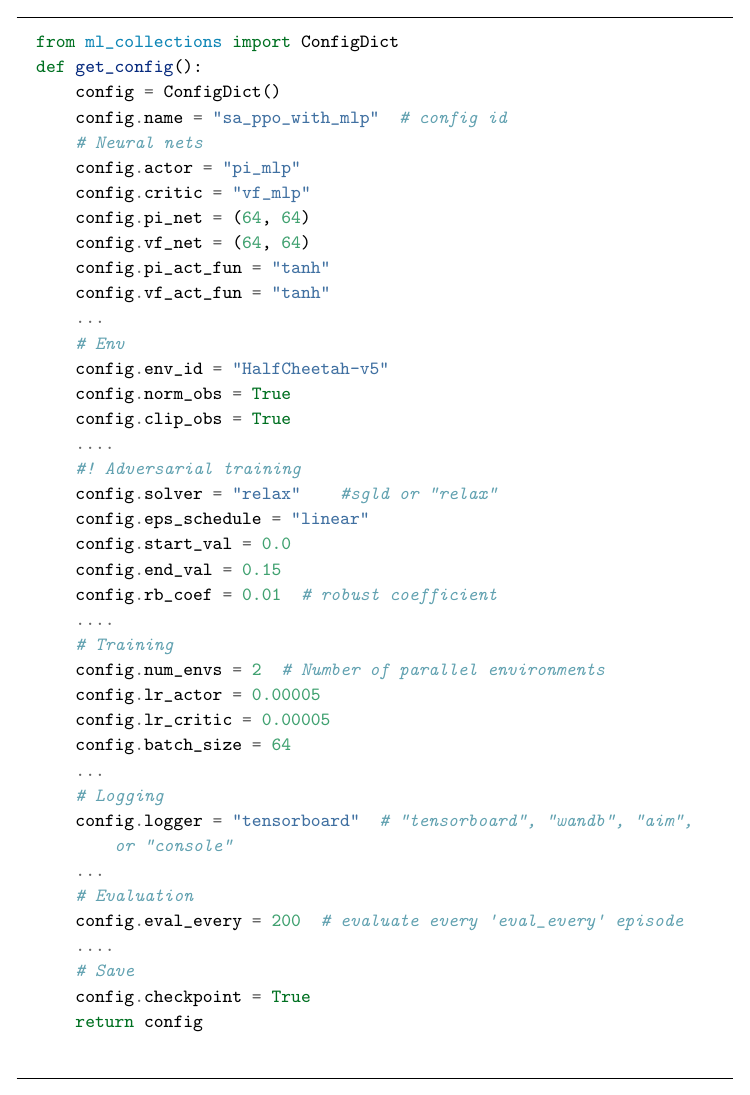}
    \caption{Configuration file of SA-Reg.}
\Description{Configuration file of SA-Reg.}
    \label{fig:config-file}
    \end{figure}

\begin{figure}[!htbp]
    \centering
    \includegraphics[width=0.85\textwidth]{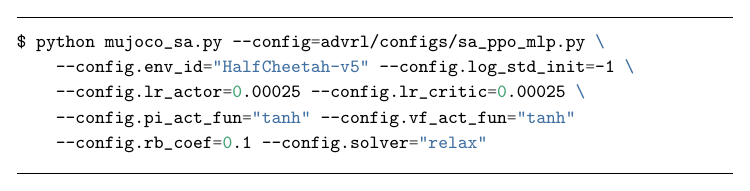}
    \caption{Command to train SA-Reg.}
\Description{Command to train SA-Reg.}
    \label{fig:training_command}
    \end{figure}

\subsubsection{Adversarial attacks}

Adversarial attacks in \texttt{advrl} follow the same configuration style as defense algorithms. Each attack is specified through a dedicated configuration file, which defines the attack parameters, evaluation settings, and logging options. Learned attacks such as RS, SA-RL and PA-AD rely on training a Q-network or an adversarial policy. For these methods, the configuration file additionally includes training hyperparameters. Figure~\ref{fig:attack_config_file} shows examples of config file of MAD and RS attacks.

\begin{figure}[!htbp]
\centering
\begin{subfigure}[b]{0.48\textwidth}
    \centering
    \includegraphics[width=\textwidth]{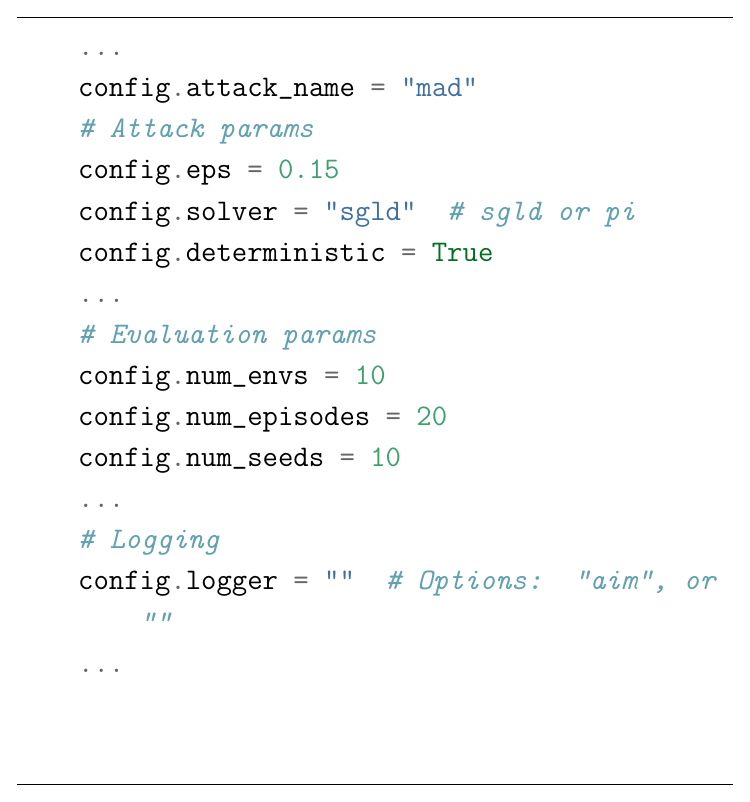}
    \caption{MAD config.}
\Description{MAD config.}
\end{subfigure}
\hfill
\begin{subfigure}[b]{0.48\textwidth}
    \centering
    \includegraphics[width=\textwidth]{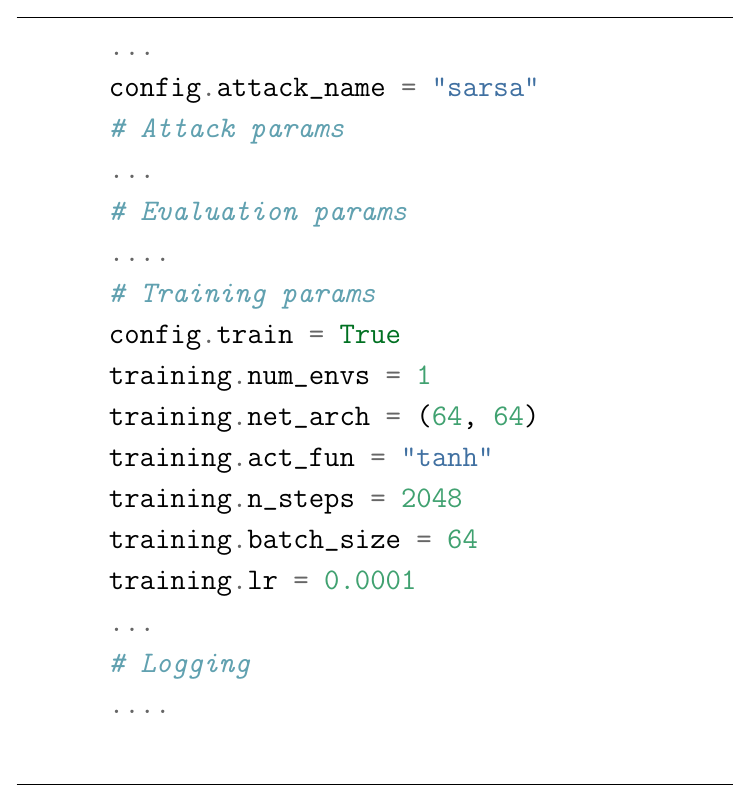}
    \caption{RS config.}
\Description{RS config.}
\end{subfigure}
\caption{Configuration files of (a) MAD and (b) RS attacks.}
\Description{Configuration files of (a) MAD and (b) RS attacks.}
\label{fig:attack_config_file}
\end{figure}

Attacks can be executed via the command line by specifying a configuration file and overriding hyperparameters if necessary. Figure~\ref{fig:run_mad} and \ref{fig:run_rs} shows example commands for running MAD and RS attacks respectively.

\begin{figure}[!htbp]
\centering
\includegraphics[width=0.85\textwidth]{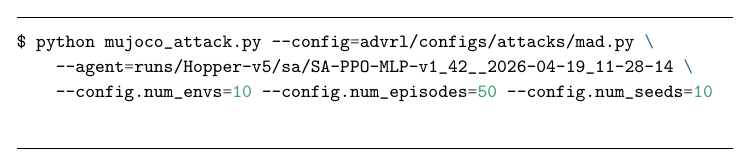}
\caption{Command to run MAD attack.}
\Description{Command to run MAD attack.}
\label{fig:run_mad}
\end{figure}

\begin{figure}[!htbp]
\centering
\includegraphics[width=0.85\textwidth]{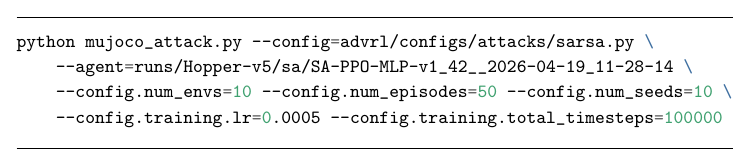}
\caption{Command to run RS attack.}
\Description{Command to run RS attack.}
\label{fig:run_rs}
\end{figure}

\subsubsection{Checkpointing and Reproducibility}

For adversarial training, \texttt{advrl} saves a checkpoint under a directory path of the form:

\begin{center}
    \texttt{runs/env\_id/algorithm/run\_name}
\end{center}

Each checkpoint contains neural network weights, list of environment wrappers, the full environment normalization statistics, and the complete run configuration and launched command. When running adversarial attacks, \texttt{advrl} saves the evaluation results under a directory path of the form \texttt{evals/env\_id/algorithm/run\_name/run\_id}. All the random seeds used across evaluation runs are stored in \texttt{seeds.json}. The complete per-seed, per-episode rewards are saved in \texttt{metadata.json}, preserving the full evaluation trajectory for traceability. The \texttt{metrics.json} file contains aggregate statistics computed across all seeds and episodes, including the mean, median, interquartile mean (IQM), interquartile range (IQR), standard deviation, minimum, and maximum returns. Full details are in table \ref{tab:checkpoint}.

\begin{table}[!htbp]
    \centering
    \caption{Files saved by \texttt{advrl} for defensive training and adversarial attack evaluation.}
    \label{tab:checkpoint}
    \label{tab:checkpoint-detail}
    \begin{tabular}{llp{7cm}}
    \toprule
    \textbf{Context} & \textbf{File} & \textbf{Description} \\
    \midrule
    \multirow{5}{*}{Training}
     &\texttt{policy.pt / \texttt{critic.pt}} & Neural Network weights \\
     & \texttt{envparams.npz}  & Environment normalization statistics (running mean, variance, count). \\
     & \texttt{wrappers.json}  & List of active environment wrappers. \\
     & \texttt{config.json}    & Full training configuration. \\
     & \texttt{command.txt}    & Exact command used to launch the experiment. \\
    \midrule
    \multirow{5}{*}{Attack}
     & \texttt{agent.json}     & Configuration of the victim policy under attack. \\
     & \texttt{config.json}    & Configuration of the attack. \\
     & \texttt{seeds.json}     & List of random seeds used across evaluation runs. \\
     & \texttt{metadata.json}  & Per-seed, per-episode rewards for the full evaluation. \\
     & \texttt{metrics.json}   & Aggregate statistics: mean, median, IQM, IQR, std, min, max. \\
    \bottomrule
    \end{tabular}
\end{table}

\section{Implementation details }
\label{sec:suppl_impl_det}
In this section we will present the implementation details of each algorithm. We will provide how we implemented each algorithm, what are the design choices we made, and how our implementation is different than the one in the previous codebases. In the following we will assume that the policy is a multivariable Gaussian distribution. The policy network output the vector mean $\mu_{\theta}(s)$ and a diagonal covariance matrix $\Sigma_{\theta}$ independent of the input $s$.

\begin{algorithm}[!htbp]
\caption{Proximal Policy Optimization (PPO)}
\KwIn{Total steps $N$, rollout length $T$, epochs $K$, mini-batch size $M$, policy network $\pi_\theta$, value network $V_\phi$.}
\BlankLine
$n \gets 0$, \;
\While{$n < N$}{

\DataCollection{
    Collect rollout $\mathcal{D} = \{(s_t, a_t, r_t, s_{t+1})\}_{t=1}^{T}$ using $\pi_{\theta}$, and update $n \gets n + T$ 
    
    Compute rewards-to-go $\hat{R}_t$ and advantages $\hat{A}_t$
    
}

\For{$\text{epoch} = 1, \dots, K$}{
    Sample mini-batch $\mathcal{B} \subset \mathcal{D}$ of size $M$\\
\PolicyUpdate{
        Update policy network using:
        \BlankLine
        $\mathcal{L}_{\text{actor}}(\theta) = - \sum_{\mathcal{B}} \min\!\left(r_t(\theta)\,\hat{A}_t,\operatorname{clip}(r_t(\theta),\, 1 \pm \epsilon)\,\hat{A}_t\right)$
    }

\CriticUpdate{
        Update value network using: 
        $\mathcal{L}_{\text{critic}}(\phi)=\sum_{\mathcal{B}}(V_\phi(s_t)-\hat{R}_t)^2$
    }
}
}
\label{algo:ppo}
\end{algorithm}

\subsection{SA-Reg}
SA-Reg was independently proposed in two works \cite{DBLP:conf/icml/ShenLJWZ20} and \cite{DBLP:conf/nips/0001CX0LBH20}.  The goal of \cite{DBLP:conf/icml/ShenLJWZ20} is to learn smooth policies with respect to states, while \cite{DBLP:conf/nips/0001CX0LBH20} provides a theoretical justification by showing that the performance degradation under an optimal adversary $v^*$ is bounded by this discrepancy: 
\begin{equation}
    \max_{s \in \mathcal{S}} \left[ V_{\pi}(s) - V_{\pi \circ v^*}(s) \right]
    \;\leq\;
    \alpha \, \max_{s \in \mathcal{S}} \; \max_{\tilde{s} \in \mathcal{B}_\epsilon(s)}
    D_{\mathrm{TV}}\!\left( \pi(.| s) \,\|\, \pi(.| \tilde{s}) \right)
\label{eq:sa-reg-ineq}
\end{equation}
where $D_{\mathrm{TV}}$ denotes the total variation distance and $\alpha$ is a constant independent of the policy $\pi$. Thus, minimizing the RHS of Eq. \eqref{eq:sa-reg-ineq} reduces the performance gap between the clean and adversarial settings.

In practice, the Kullback–Leibler (KL) divergence is used and the loss is computed over a batch of sampled states:
\begin{equation}
\label{eq:sa-reg-batch-loss}
\mathcal{L}_{\mathrm{robust}}^{\mathrm{SA-Reg}} = \sum_{s} \max_{\tilde{s_t} \in \mathcal{B}_{\epsilon}(s_t)} \mathcal{D}_{\mathrm{KL}}(\pi(.|s), \pi(.|\tilde{s}))
\end{equation}

To compute the inner maximization $\max_{\tilde{s} \in \mathcal{B}_\epsilon(s)}$, we can either us convex relaxation tools \cite{DBLP:conf/nips/XuS0WCHKLH20} or Stochastic Gradient Langevin Dynamics (SGLD) \cite{DBLP:conf/nips/0001CX0LBH20}.
The full pseudocode of SA-Reg as we implemented it is in Algorithm 2. Compared to vanilla PPO in Algorithm 1,  SA-Reg modifies the policy update step by adding a regularization term. The authors in \cite{DBLP:conf/nips/0001CX0LBH20} proposed two methods to compute $\max_{\tilde{s} \in \mathcal{B}_{\epsilon}(s)} \mathcal{D}_{\mathrm{KL}}(\pi(.|s), \pi(.|\tilde{s}))$: convex-relaxation tools and SGLD. In addition to these two approaches, we added power iteration (PI) method \cite{DBLP:journals/pami/MiyatoMKI19}. In the following we will explain each method.

\subsubsection*{Convex relaxation tools:}
Given the input state $s$ and the current value of attack budget $\epsilon_t$. we use \texttt{auto\_LIRPA} \cite{DBLP:conf/nips/XuS0WCHKLH20} library to get the lower and upper bounds of the mean vector $\underline{\mu}_{\theta}(s)$ and $\overline{\mu}_{\theta}(s)$. 
\begin{equation}
\max_{\tilde{s} \in \mathcal{B}_{\epsilon}(s)} \mathcal{D}_{\mathrm{KL}}(\pi(.|s), \pi(.|\tilde{s})) \approx \frac{1}{2} \sum_{i=1}^{\mathcal{A}} \frac{\max[\underline({\mu}_{\theta}(s) - \mu_{\theta}(s))^2,(\overline{\mu}_{\theta}(s) - \mu_{\theta}(s))^2]_{i}}{\Sigma_{ii}}
\end{equation}

\subsubsection*{Stochastic Gradient Langevin Dynamics}

SGLD solves the problem in \ref{eq:sa-reg-max-problem} using an iterative approach. SGLD \cite{DBLP:conf/colt/Lamperski21} is similar to gradient descent but it adds a noise factor to the update to escape the zero gradient at $\tilde{s}_0 = s$. 
\begin{equation}
\label{eq:sa-reg-max-problem}
\max_{\tilde{s} \in \mathcal{B}_{\epsilon}(s)} \mathcal{D}_{\mathrm{KL}}(\pi(.|s), \pi(.|\tilde{s}))  = \max_{\tilde{s} \in \mathcal{B}_{\epsilon}(s)} 0.5 \sum_{i=1}^{\mathcal{A}} \frac{[\mu_{\theta}(\tilde{s}) - \mu_{\theta}(s)]_{i}^2}{\Sigma_{ii}} = \max_{\tilde{s} \in \mathcal{B}_{\epsilon}(s)} \mathcal{R}(\tilde{s})
\end{equation}

\begin{equation}
\tilde{s}_{k+1} = \text{proj}\left(s_k + \eta \nabla_{\tilde{s}_k}\mathcal{R}(\tilde{s}_k) + \sqrt{\frac{2 \eta}{\beta}}\zeta \right), \tilde{s}_0 = s, k = 0, \dots K-1
\end{equation}
Here $\zeta  \sim\mathcal{N}(0,1) $, $\eta = \frac{\epsilon_t}{K}$ is the step size parameter and $\beta > 0$ is a noise parameter.

The original codebase of SA-Reg has a bug in the way it implemented SGLD. SGLD was implemented as in Eq \ref{eq:old-sa-reg-sgld}. As most previous codebases are built on top of SA-Reg codebase, they all inherited this bug.

\begin{equation}
\label{eq:old-sa-reg-sgld}
\tilde{s}_{k+1} = \text{proj}\left(s_k + \eta \nabla_{\tilde{s}_k}\mathcal{R}(\tilde{s}_k) + \textcolor{red}{\eta \times}\sqrt{\frac{2 \eta}{\beta}}\zeta \right), \tilde{s}_0 = s, k = 0, \dots K-1
\end{equation}

Given that SA-Reg uses an $\epsilon$-scheduler during training, where the maximum size of the perturbations slowly increases until it reaches $\epsilon$, this leads to having very small step sizes in the early stages of training. In that case, SGLD is nearly equivalent to gradient descent which is problematic given that the gradient at the first step is zero $\nabla_{\tilde{s}_0 = s}\mathcal{R}(\tilde{s}_0) = 0$.

\subsubsection*{Power Iteration method}
To solve the maximization problem in \ref{eq:sa-reg-max-problem} using PI \cite{DBLP:journals/pami/MiyatoMKI19} we do the following:

\begin{equation}
\max_{\tilde{s} \in \mathcal{B}_{\epsilon}(s)} \mathcal{D}_{\mathrm{KL}}(\pi(.|s), \pi(.|\tilde{s})) = \mathcal{R}(s+ \epsilon d_{K}) 
\end{equation}
Where:
\begin{equation}
d_{k+1} = \text{sign}\left(\nabla_{\xi d_k} \mathcal{R}(s + d_k) \right), d_0 \sim\mathcal{N}(0,1), k=0, \dots, K-1 
\end{equation}

Here $\xi$ is a small number (e.g., $10^{-3}$, $10^{-5}$).PI has the advantage of needing only one iteration to have good performance.

In figure \ref{fig:sgld_pi} we keep track of the estimations of PI with 1 iteration and SGLD with 10 iterations when training using convex relaxation. There is a huge gap between convex relaxations and SGLD. PI gives a better estimation while using just 1 iteration.
\begin{figure}[!htbp]
\centering
\includegraphics[width=0.45\textwidth]{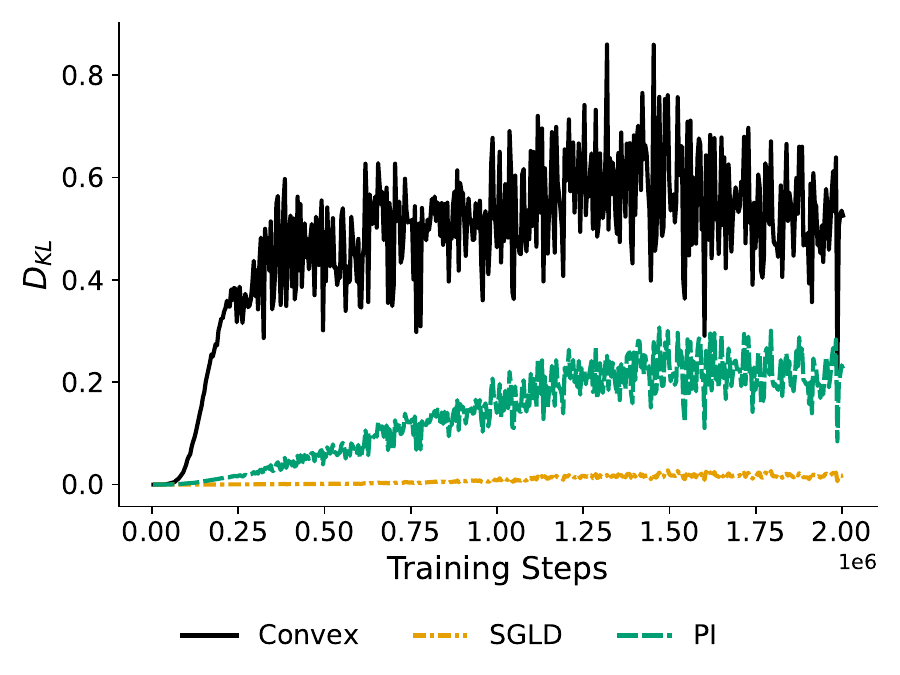}
\caption{Comparison between convex relaxation, PI and SGLD.}
\Description{Comparison between convex relaxation, PI and SGLD.}
\label{fig:sgld_pi}
\end{figure}

\begin{algorithm}[!htbp]
\caption{SA-Reg}
\KwIn{Total steps $N$, rollout length $T$, epochs $K$, mini-batch size $M$, policy network $\pi_\theta$, value network $V_\phi$, \textcolor{teal}{ attack budget $\epsilon$, robustness coefficient $\kappa$}.}
\BlankLine
$n \gets 0$ \\
\textcolor{teal}{ Initialize the $\epsilon$-scheduler $\epsilon_t(\epsilon,n)$}. \\
\While{$n < N$}{

    \DataCollection{
        Collect rollout $\mathcal{D} = \{(s_t, a_t, r_t, s_{t+1})\}_{t=1}^{T}$ using $\pi_{\theta}$, and update $n \gets n + T$ 
        
        Compute rewards-to-go $\hat{R}_t$ and advantages $\hat{A}_t$
        
    }
    \textcolor{teal}{Update the $\epsilon$-scheduler $\epsilon_t(\epsilon,n)$} \\
    \For{$\text{epoch} = 1, \dots, K$}{
        Sample mini-batch $\mathcal{B} \subset \mathcal{D}$ of size $M$\\
        
    \textcolor{teal}{
    \PolicyUpdate{
        Compute the regularization loss (convex-relax, SGLD, or PI):\\
            \BlankLine
            $\mathcal{L}_{\mathrm{robust}}(\theta) = \sum_{s \in \mathcal{B}} \max_{\tilde{s} \in \mathcal{B}_{\epsilon}(s)} \mathcal{D}_{\mathrm{KL}}(\pi(.|s), \pi(.|\tilde{s}))$ \\
            Compute standard PPO loss:
            \BlankLine
            $\mathcal{L}_{\text{DRL}}(\theta) = -\sum_{\mathcal{B}} \min\!\left(r_t(\theta)\,\hat{A}_t,\operatorname{clip}(r_t(\theta),\, 1 \pm \epsilon)\,\hat{A}_t\right)$ \\
            Update policy network using:
            \BlankLine
            $\mathcal{L}(\theta) = \mathcal{L}_{\text{DRL}}(\theta) + \kappa \mathcal{L}_{\mathrm{robust}}(\theta)$ \\
        }
    }
    
    \CriticUpdate{
            Update value network using: 
            $\mathcal{L}^{\text{critic}}(\phi)=\sum_{\mathcal{B}}(V_\phi(s_t)-\hat{R}_t)^2$
        }
    }
}
\label{algo:sa-reg}
\end{algorithm}

\subsection{Radial}

Radial original codebase used a different robustness loss than the one reported in the original paper. The paper reports the following loss: 

\begin{equation}
\mathcal{L}^\mathrm{Radial}(\theta) = \mathbb{E}_{(s_t,a_t)} \left[ -
\min \left(
\frac{\pi^{\mathrm{wc}}_\theta(a_t \mid s_t)}{\pi_{\theta_{\text{old}}}(a_t \mid s_t)} A_t,
\mathrm{clip}\big(\frac{\pi^{\mathrm{wc}}_\theta(a_t \mid s_t)}{\pi_{\theta_{\text{old}}}(a_t \mid s_t)}, 1 - \epsilon, 1 + \epsilon\big)A_t
\right)
\right]
\end{equation}

However, the codebase used the following loss:

\begin{equation}
\begin{aligned}
\mathcal{L}^\mathrm{Radial}(\theta)
= \mathbb{E}_{(s_t,a_t)} \Bigg[
- \min \Big(
&\underline{r}(\theta) A_t,\mathrm{clip}\big(\underline{r}(\theta), 1 \pm \epsilon\big) A_t, \\
&\overline{r}(\theta) A_t, 
\mathrm{clip}\big(\overline{r}(\theta), 1 \pm \epsilon\big) A_t, \\
&r(\theta) A_t, 
\mathrm{clip}\big(r(\theta), 1 \pm \epsilon\big) A_t
\Big)
\Bigg]
\end{aligned}
\end{equation}
Moreover, the policy network is updated using a convex combination between the standard DRL loss and the Radial loss :
\begin{equation}
\mathcal{L}(\theta) = \kappa\mathcal{L}_{\mathrm{DRL}}(\theta) + (1-\kappa) \mathcal{L}_{\mathrm{robust}}(\theta),  \quad \kappa \in (0,1]
\end{equation}

We report the pseudocode of Radial in Algorithm 3. We implemented the loss as the one in the original paper. And users can choose if they want to use convex combination loss or not.

\begin{algorithm}[!htbp]
\caption{Radial}
\KwIn{Total steps $N$, rollout length $T$, epochs $K$, mini-batch size $M$, policy network $\pi_\theta$, value network $V_\phi$, \textcolor{teal}{ attack budget $\epsilon$, robustness coefficient $\kappa$}.}
\BlankLine
$n \gets 0$ \\
\textcolor{teal}{ Initialize the $\epsilon$-scheduler $\epsilon_t(\epsilon,n)$}. \\
\While{$n < N$}{

    \DataCollection{
        Collect rollout $\mathcal{D} = \{(s_t, a_t, r_t, s_{t+1})\}_{t=1}^{T}$ using $\pi_{\theta}$, and update $n \gets n + T$ 
        
        Compute rewards-to-go $\hat{R}_t$ and advantages $\hat{A}_t$
        
    }
    \textcolor{teal}{Update the $\epsilon$-scheduler $\epsilon_t(\epsilon,n)$} \\
    \For{$\text{epoch} = 1, \dots, K$}{
        Sample mini-batch $\mathcal{B} \subset \mathcal{D}$ of size $M$\\
        
    \textcolor{teal}{
    \PolicyUpdate{
        Compute the worst case policy $\pi_{\text{wc}}$, and compute the regularization loss:\\
        \BlankLine
        $\mathcal{L}_{\text{robust}}(\theta) = - \sum_{\mathcal{B}} \min\!\left(r^{\text{wc}}_t(\theta)\,\hat{A}_t,\operatorname{clip}(r^{\text{wc}}_t(\theta),\, 1 \pm \epsilon)\,\hat{A}_t\right)$ \\
        Compute standard PPO loss:
        \BlankLine
        $\mathcal{L}_{\text{DRL}}(\theta) = -\sum_{\mathcal{B}} \min\!\left(r_t(\theta)\,\hat{A}_t,\operatorname{clip}(r_t(\theta),\, 1 \pm \epsilon)\,\hat{A}_t\right)$ \\
        Update policy network using:
        \BlankLine
        $\mathcal{L}(\theta) = \mathcal{L}_{\text{DRL}}(\theta) + \beta \mathcal{L}_{\mathrm{robust}}(\theta)$ \\
        }
    }
    
    \CriticUpdate{
            Update value network using: 
            $\mathcal{L}^{\text{critic}}(\phi)=\sum_{\mathcal{B}}(V_\phi(s_t)-\hat{R}_t)^2$
        }
    }
}
\label{algo:radial}
\end{algorithm}

\subsection{SA-ATLA}
The Alternating Training with Learned Adversary pseudocode is in Algorithm \ref{algo:atla}. In our implementation of SA-ATLA, we provide two design for the adversarial policy:
\begin{itemize}
\item Squashed: Wa add a Tanh activation function at the last layer of the policy network so the outputs of the policy are in $[-1,1]$. The perturbed states is $\tilde{s}_t = s_t + \epsilon_t \times v(s_t)$. We correct the probabilities.
\item No squashing: We do not constrain the output of the policy, but we constrain the perturbation to not exceed the attack budget.
\end{itemize}

Moreover, use use $\epsilon$-scheduling to collect rollouts to train the robust agent. The adversary is instead trained with the full attack budget. The scheduling is important to ensure the convergence of the agent. The original implementation did not use scheduling.  

\begin{algorithm}[!htbp]
\caption{ATLA}
\label{algo:atla}
\KwIn{Total steps $N$, rollout length $T$, epochs $K$, mini-batch size $M$, policy network $\pi_\theta$, value network $V_\phi$, \textcolor{teal}{ adversary policy network $v_{\theta^{\mathrm{adv}}}$, adversary value network $V_{\phi^{\mathrm{adv}}}$,attack budget $\epsilon$, robustness coefficient $\kappa$}.}
\BlankLine
$n \gets 0$, \textcolor{teal}{$n^{\mathrm{adv}} \gets 0$} \\
\textcolor{teal}{ Initialize the $\epsilon$-scheduler $\epsilon_t(\epsilon,n)$}. \\
\While{$n < N$ and \textcolor{teal}{$n^{\mathrm{adv}} < N$}}{
    \textcolor{teal}{
        \TrainAdversary{
        \DataCollection{
            Collect rollout $\mathcal{D}^{\mathrm{adv}} = \{(s_t, v_{\theta^{\mathrm{adv}}}(s_t), -r_t, s_{t+1})\}_{t=1}^{T}$ using $v_{\theta^{\mathrm{adv}}}$ as perturbation of size $\epsilon$  and $\pi_{\theta}$ to step the environment, and update $n^{\mathrm{adv}} \gets n^{\mathrm{adv}} + T$ \\
            Compute rewards-to-go $\hat{R}_t^{\mathrm{adv}}$ and advantages $\hat{A}_t^{\mathrm{adv}}$}
        Update the $\epsilon$-scheduler $\epsilon_t(\epsilon,n)$ \\
        \For{$\text{epoch} = 1, \dots, K$}{
            Sample mini-batch $\mathcal{B} \subset \mathcal{D}^{\mathrm{adv}}$ of size $M$\\
            \PolicyUpdate{
                    Update policy network using:
                    \BlankLine
                    $\mathcal{L}_{\text{actor}}(\theta^{\mathrm{adv}}) = -\sum_{\mathcal{B}} \min\!\left(r_t(\theta^{\mathrm{adv}})\,\hat{A}_t^{\mathrm{adv}},\operatorname{clip}(r_t(\theta^{\mathrm{adv}}),\, 1 \pm \epsilon)\,\hat{A}_t^{\mathrm{adv}}\right)$}
            \CriticUpdate{
                    Update value network using: 
                    $\mathcal{L}_{\text{critic}}(\phi^{\mathrm{adv}})=\sum_{\mathcal{B}}(V_\phi^{\mathrm{adv}}(s_t)-\hat{R}_t^{\mathrm{adv}})^2$}
        }
    }
    }
    \DataCollection{
        \textcolor{teal}{Collect rollout under attack $\mathcal{D} = \{(v_{\theta^{\mathrm{adv}}}(s_t), a_t, r_t, v_{\theta^{\mathrm{adv}}}(s_{t+1}))\}_{t=1}^{T}$ using $\pi_{\theta}$, and update $n \gets n + T$} 
        
        Compute rewards-to-go $\hat{R}_t$ and advantages $\hat{A}_t$}
    \textcolor{teal}{Update the $\epsilon$-scheduler $\epsilon_t(\epsilon,n)$} \\
    \For{$\text{epoch} = 1, \dots, K$}{
        Sample mini-batch $\mathcal{B} \subset \mathcal{D}$ of size $M$\\
    \PolicyUpdate{
            Update policy network using:
            \BlankLine
            $\mathcal{L}_{\text{actor}}(\theta) = -\sum_{\mathcal{B}} \min\!\left(r_t(\theta)\,\hat{A}_t,\operatorname{clip}(r_t(\theta),\, 1 \pm \epsilon)\,\hat{A}_t\right)$
        }

    \CriticUpdate{
            Update value network using: 
            $\mathcal{L}_{\text{critic}}(\phi)=\sum_{\mathcal{B}}(V_\phi(s_t)-\hat{R}_t)^2$
        }
    }
}
\end{algorithm}

\subsection{PA-ATLA}
PA-ATLA follows the same training framework as SA-ATLA. The difference is that the adversary learns the worst-action, which is then used to conducted a targeted attacking using FGSM. To conduct the FGSM, we use the adversarial losses recommended in the original paper for continuous actions (Section C.3 in \cite{DBLP:conf/iclr/SunZLH22}). The original paper used the following to generate the perturbations:
\begin{equation}
\argmax_{\tilde{s} \in \mathcal{B}_{\epsilon}(s)} \| \pi(\tilde{s}) - \pi(s)\|_{2} +\text{CosineSim}(\pi(\tilde{s}) - \pi(s),\hat{a})
\end{equation}
This loss is was recommended for discrete actions. \texttt{advrl} uses the following loss against deterministic policy. We also support PGD generation. 

\begin{equation}
\argmax_{\tilde{s} \in \mathcal{B}_{\epsilon}(s)} \| \pi(\tilde{s}) - \pi(s)\|_{2}
\end{equation}
\subsection{WocaR-RL}
The pseudocode of WocaR-RL is provided in Algorithm 5. The worst-case critic network is updated using a loss similar to Double Q-learning error : q-network $\underline{Q}_{\psi}$ is used to select the minimum action and evaluated using the target q-network:

\[\mathcal{L}_{\text{worst-case}}(\psi) = \left(r_t + \gamma Q_{\psi^{-}}(s_{t+1},\argmin_{a \in \mathcal{A}_{\mathrm{adv}}(s_{t+1}, \pi)} \underline{Q}_{\psi}(s_{t+1},a)) - Q_{\psi}(s_t,a_t) \right)^2
\] 

In the original implementation, $\mathcal{A}_{\mathrm{adv}}(s,\pi) = [\underline{\mu}_{\theta}(s), \overline{\mu}_{\theta}(s) ]$ is computed using Interval Bound Propagation (IBP) with a fixed perturbation $\epsilon=0.01$ across all the environments. The minimum is computed using 50-step PGD with a step size of $ (\underline{\mu}_{\theta}(s) - \overline{\mu}_{\theta}(s))/50$. We follow the same logic, except that (1) the user can use IBP, CROWN or CROWN-IBP (2) the perturbations is not fixed at $0.01$ and we use $\epsilon_t$-scheduling. 

\begin{algorithm}[!htbp]
\caption{WocaR-RL}
\KwIn{Total steps $N$, rollout length $T$, epochs $K$, mini-batch size $M$, policy network $\pi_\theta$, value network $V_\phi$, \textcolor{teal}{attack budget $\epsilon$, robustness coefficients $\kappa_1$ and $\kappa_2$, worst-case critic $\underline{Q}_{\psi}$ and target worst-case critic $\underline{Q}_{\psi^{-}}$ }}
\BlankLine
$n \gets 0$, \\
\textcolor{teal}{ Initialize the $\epsilon$-scheduler $\epsilon_t(\epsilon,n)$}\\
\While{$n < N$}{

    \DataCollection{
        Collect rollout $\mathcal{D} = \{(s_t, a_t, r_t, s_{t+1})\}_{t=1}^{T}$ using $\pi_{\theta}$, and update $n \gets n + T$ 
        
        Compute rewards-to-go $\hat{R}_t$ and advantages $\hat{A}_t$
        
    }
    \textcolor{teal}{
    \WorstCaseCritic{
        Compute $\mathcal{A}_\mathrm{adv}$: lower $\underline{\pi}_\theta$ and upper $\overline{\pi}_\theta$  bound of the policy with pertubations of size $\epsilon_t$ \\
        Compute target actions $\hat{a}$ using PGD:
        \[\hat{a} = \argmin_{a \in \mathcal{A}_{\mathrm{adv}}(s_{t+1}, \pi)} \underline{Q}_{\psi}(s_{t+1},a)\] 
        Update the worst-case critic using: $\mathcal{L}_{\text{worst-case}}(\psi) =\sum_{\mathcal{B}} \left(r_t + \gamma Q_{\psi^{-}}(s_{t+1},\hat{a}_{t+1}) - Q_{\psi}(s_t,a_t)\right)^2 $ \\
        Periodically update the target worst-case critic.\\
    }}
    \For{$\text{epoch} = 1, \dots, K$}{
        Sample mini-batch $\mathcal{B} \subset \mathcal{D}$ of size $M$\\
    \textcolor{teal}{
    \PolicyUpdate{
        Compute state importance weights using PGD:
        \[ w(s) = V_{\phi}(s) - \min_{a \in \mathcal{A}}\underline{Q}_{\psi}(s,a)\] 
        Compute the weighted SA-Reg loss:
        $\mathcal{L}_{\mathrm{SA-Reg}}(\theta) = \sum_{s \in \mathcal{B}} w(s)  \max_{\tilde{s} \in \mathcal{B}_{\epsilon}(s)} \mathcal{D}_{\mathrm{KL}}(\pi(.|s), \pi(.|\tilde{s}))$ \\
        Compute the robust PPO loss:
        \BlankLine
        $\mathcal{L}_{\text{wc}}(\theta) = -\sum_{\mathcal{B}} \min\!\left(r_t(\theta)\,(\hat{A}_t + \kappa_1 \underline{Q}_{\psi}(s_t,a_t) ),\operatorname{clip}(r_t(\theta),\, 1 \pm \epsilon)\,(\hat{A}_t + \kappa_1 \underline{Q}_{\psi}(s_t,a_t) )\right)$\\
        Update policy network using:
        $\mathcal{L}(\theta) = \mathcal{L}_{\text{wc}}(\theta) + \kappa_2 \mathcal{L}_{\mathrm{SA-Reg}}(\theta) $ 
        }}

    \CriticUpdate{
            Update value network using: 
            $\mathcal{L}_{\text{critic}}(\phi)=\sum_{\mathcal{B}}(V_\phi(s_t)-\hat{R}_t)^2$
        }
    }
}
\label{algo:wocar}
\end{algorithm}

\subsection{Advis}
\label{sec:supp_advis}

The pseudocode of Advis is presented in Algorithm 6. We implemented three importance sampling techniques that have lower variance that the ordinary importance sampling estimator: Per-Decision IS (PDIS) (Eq \ref{eq:supp_pdis}), Weighted IS (WIS) (Eq \ref{eq:supp_wis}), and Weighted Per-Decision IS (WPDIS) (Eq \ref{eq:supp_wpdis}).
\begin{equation}
\hat{J}_{\mathrm{PDIS}}(\pi) = \frac{1}{N} \sum_{i=1}^{N} \sum_{t=0}^{T-1} \rho_t(\tau_i)\, \gamma^t r_t^i
\label{eq:supp_pdis}
\end{equation}
\begin{equation}
\hat{J}_{\mathrm{WIS}}(\pi) = \sum_{i=1}^{N} \frac{\rho_{T-1}(\tau_i)}{\sum_{j=1}^{N} \rho_{T-1}(\tau_j)} R(\tau_i) 
\label{eq:supp_wis}
\end{equation}
\begin{equation}
\hat{J}_{\mathrm{WPDIS}}(\pi) = \sum_{t=0}^{T-1} \gamma^t \sum_{i=1}^{N} \frac{\rho_t(\tau_i)}{\sum_{j=1}^{N} \rho_t(\tau_j)} r_t^i
\label{eq:supp_wpdis}
\end{equation}

In contrast to previous algorithms, Advis collect full episodes rather than a fixed number of steps. The training is done using mini- batches: consecutive windows of the collected episodes. We provide two options to compute the importance weights: (1) importance weights only consider the probabilities of states in the sampled mini-batch, or (2) the importance weights of the current mini-batch are initialized using importance weights from the previous mini-batch, except for the first mini-batch. However, the gradients are propagated only through the probabilities of the sampled window. We additionally give the option to use average rewards ($\gamma = 1$), as it may otherwise cause estimates to vanish to zero in environments with long episodes (e.g., $0.99^{999} \approx 4.3 \times 10^{-5}$).

To ensure numerical stability, we can clip the log-importance weights when using PDIS or WPDIS. For WIS estimators, we compute normalized importance weights from the log-importance weights log-sum-exp trick for numerical stability. 
\begin{algorithm}[thb]
\caption{Advis}
\KwIn{Total steps $T$, \textcolor{teal}{number of episodes $N$}, epochs $K$, mini-batch size $M$, policy network $\pi_\theta$, value network $V_\phi$, \textcolor{teal}{ attack budget $\epsilon$, robustness coefficient $\kappa$}.}
\BlankLine
$t \gets 0$ \\
\textcolor{teal}{ Initialize the $\epsilon$-scheduler $\epsilon_t(\epsilon)$}. \\
\While{$t < T$}{
    \textcolor{teal}{
    \DataCollection{
        Collect rollout $N$ episodes $\mathcal{D} = \{(s^{i}_0,a^{i}_o, \dots, s^{i}_t, a^{i}_t, r^{i}_t,\dots, s^{i}_{L_i})\}_{i=1}^{N}$ using $\pi_{\theta}$, and update $t \gets t + \sum_i L_i$ \\
        Compute rewards-to-go $\hat{R}_t$ and advantages $\hat{A}_t$
    }}
    \textcolor{teal}{Update the $\epsilon$-scheduler $\epsilon_t(\epsilon)$} \\
    \For{$\text{epoch} = 1, \dots, K$}{
        Sample mini-batch $\mathcal{B} \subset \mathcal{D}$ of size $M$\\
        
    \textcolor{teal}{
    \PolicyUpdate{
        Compute the worst case policy $\pi_{\text{wc}}$ using Radial closed form\\
        Compute worst-case performance using importance sampling $\hat{J}_{\mathrm{IS}}(\pi_{\mathrm{wc}})$\\
        Compute adversarial loss $\mathcal{L}_{\mathrm{robust}}^{\mathrm{advis}}$ \\
        Compute standard PPO loss:
        \BlankLine
        $\mathcal{L}_{\text{DRL}}(\theta) = -\sum_{\mathcal{B}} \min\!\left(r_t(\theta)\,\hat{A}_t,\operatorname{clip}(r_t(\theta),\, 1 \pm \epsilon)\,\hat{A}_t\right)$ \\
        Update policy network using:
        \BlankLine
        $\mathcal{L}(\theta) = \mathcal{L}_{\text{DRL}}(\theta) + \beta \mathcal{L}_{\mathrm{robust}}^{\mathrm{advis}}$ \\
        }
    }
    
    \CriticUpdate{
            Update value network using: 
            $\mathcal{L}^{\text{critic}}(\phi)=\sum_{\mathcal{B}}(V_\phi(s_t)-\hat{R}_t)^2$
        }
    }
}
\label{algo:advis}
\end{algorithm}

\section{Experiment details}
\subsection{Adversarial evaluation}
All the agents are trained for 2 million steps. We use the same network architecture across all the algorithms and environments: MLP network with 2 layers, each with 64 neurons. We use Adam optimizer. We use the same $\epsilon$-scheduler for all the agents: we slowly increase $\epsilon_t$ from zero until it reaches the full attack budget $\epsilon$ after 75\% of the total training steps. For the baselines, training is done using 2 environment, while Advis uses 5 environments. 

For SA-Reg, the grid search considers the following hyperparameters:
\[
\begin{aligned}
\text{Robustness coefficient }\kappa &\in \{0.001, 0.01, 0.05, 0.1\},\\
\text{Actor and critic learning rate} &\in \{0.0001, 0.00025, 0.0004\},\\
\text{Entropy coefficient} &\in \{0, 0.001\}.
\end{aligned}
\] We use the same values for Radial. For WocaR-RL, we consider  the actor, the critic, and the worst-case Q-network learning rates in $\{0.0001, 0.00025, 0.0004\}$, robust coefficient in $\{ 0.1, 0.5, 0.8\}$, soft polyak coefficient in $\{ 0.001, 0.05\}$ and entropy coefficient in $\{0, 0.001\}$.

For SA-ATLA and PA-ATLA, we consider actor and critic learning rates in $\{0.0001, 0.00025, 0.0004\}$, adversary’s actor and critic learning rates in $\{0.0001, 0.00025, 0.001\}$, entropy coefficient in $\{0, 0.001\}$, and adversary’s entropy coefficient in $\{0, 0.001\}$. For SA-ATLA we consider squashing the actions of the adversary in $\{\text{True}, \text{False}\}$.

For Advis, we use actor learning rates in  $\{ 0.00025, 0.0004, 0.0006\}$ (we need higher learning rates because Advis uses more environments that the baselines) critic learning rates in $\{0.0001, 0.00025, 0.0004\}$, robust coefficient in $\{0.01, 0.3, 0.5, 1, 1.5\}$, use discounted returns in  $\{\text{True}, \text{False}\}$, if we do not use discounted returns we add the variant where we compute the importance weights over the full trajectory.

\subsection{Adversarial evaluation}
\label{sec:detail_attacks}

\paragraph{Learned adversaries:} We perform a grid search over the hyperparameters of the learned adversaries. 
Table~\ref{tab:attack-grids} summarizes the search spaces considered for Robust Sarsa (RS), SA-RL and PA-AD. SA-RL and PA-AD adversary are trained for 2 million steps using PPO. For RS attacks, we train the for 200000 steps where the robust coefficient is zero, otherwise we train for 1 million step.

\paragraph{Random search:} For the random search used in Section 5.2.4, we sample hyperparameters from log-uniform distributions: $\texttt{lr} \sim \log\mathcal{U}(10^{-4}, 5\cdot10^{-3})$ and $\texttt{polyak} \sim \log\mathcal{U}(10^{-3}, 5\cdot10^{-2})$. The robustness coefficient is set to zero with probability $1/6$, and otherwise sampled as $\log\mathcal{U}(10^{-3}, 1)$.

\begin{table*}[htbp]
\centering
\caption{Hyperparameter search spaces for learned attacks.}
\begin{tabular}{lll}
\toprule
\textbf{Method} & \textbf{Hyperparameter} & \textbf{Values} \\
\midrule
\multirow{4}{*}{RS}
& Learning rate    & $\{1\text{e-}4,\ 2.5\text{e-}4,\ 5\text{e-}4,\ 1\text{e-}3,\ 2.5\text{e-}3,\ 5\text{e-}3\}$ \\
& Robust coefficients          & $\{0,\ 0.001,\ 0.01,\ 0.1,\ 0.5,\ 1\}$ \\
& Polyak updates           & $\{0.001,\ 0.005,\ 0.01,\ 0.05\}$ \\
\midrule
\multirow{5}{*}{SA-RL}
& Actor learning rate         & $\{1\text{e-}4,\ 5\text{e-}4,\ 1\text{e-}3,\ 2.5\text{e-}3,\ 5\text{e-}3\}$ \\
& Critic learning rate        & $\{1\text{e-}4,\ 5\text{e-}4,\ 1\text{e-}3,\ 2.5\text{e-}3,\ 5\text{e-}3\}$ \\
& Entropy coefficient    & $\{0,\ 1\text{e-}5\}$ \\
& Squash actions           & $\{\texttt{True},\ \texttt{False}\}$ \\
& Initial log standard deviation     & $\{-1,\ -2\}$ \\
\midrule
\multirow{5}{*}{PA-AD}
& Actor learning rate       & $\{1\text{e-}4,\ 5\text{e-}4,\ 1\text{e-}3,\ 2.5\text{e-}3,\ 5\text{e-}3,\ 1\text{e-}2\}$ \\
& Critic learning rate     & $\{1\text{e-}4,\ 5\text{e-}4,\ 1\text{e-}3,\ 2.5\text{e-}3,\ 5\text{e-}3,\ 1\text{e-}2\}$ \\
& Activation       & $\{\texttt{relu},\ \texttt{tanh}\}$ \\
& Entropy coefficient    & $\{0,\ 1\text{e-}3,\ 1\text{e-}4 \}$ \\
& Clip range       & $\{0.2,\ 0.3\}$ \\
\bottomrule
\end{tabular}
\label{tab:attack-grids} 
\end{table*}

\section{Additional Results}

We conduct additional experiments using another RL algorithm, an additional environment, and different attack budgets.

\paragraph{Another RL algorithm.}
We evaluate Advis using Asynchronous Advantage Actor-Critic (A3C) and compare it against SA-Reg and Radial, the two closest baselines. We evaluate against Robust SARSA (RS), one of the strongest attacks in the main results. Table~\ref{tab:additional-algorithm} reports the worst-case return, $\text{CVaR}_{0.01}$, and $\text{CVaR}_{0.05}$. Advis achieves the highest returns across all three environments and metrics.
\begin{table}[htbp]
    \centering
    \small
    \caption{Robustness under RS attacks using A3C. Returns are reported with standard deviations.}
    \label{tab:additional-algorithm}
    \begin{tabular}{llccc}
        \toprule
        Env. & Method & Minimal & $\mathrm{CVaR}_{0.01}$ & $\mathrm{CVaR}_{0.05}$ \\
        \midrule
        \multirow{3}{*}{HalfCheetah}
        & SA-Reg & $3638.24 \pm 61.47$ & $3678.44 \pm 7.8$ & $3877.19 \pm 198.86$ \\
        & Radial & $4045.41 \pm 1098.94$ & $4058.27 \pm 51.44$ & $4146.13 \pm 12.44$ \\
        & Advis & $\mathbf{4164.01 \pm 529}$ & $\mathbf{4174.62 \pm 278.46}$ & $\mathbf{4186.35 \pm 133.56}$ \\
        \midrule
        \multirow{3}{*}{Hopper}
        & SA-Reg & $992.47 \pm 104.66$ & $994.52 \pm 95.41$ & $998.55 \pm 84.86$ \\
        & Radial & $955.10 \pm 176.57$ & $979.79 \pm 55.21$ & $1016.36 \pm 46.42$ \\
        & Advis & $\mathbf{996.08 \pm 124.97}$ & $\mathbf{1008.44 \pm 102.05}$ & $\mathbf{1022.48 \pm 136.19}$ \\
        \midrule
        \multirow{3}{*}{Walker2d}
        & SA-Reg & $714.39 \pm 742.98$ & $750.26 \pm 783.12$ & $992.89 \pm 731.56$ \\
        & Radial & $1725.68 \pm 799.05$ & $1744.69 \pm 803.95$ & $1766.34 \pm 818.79$ \\
        & Advis & $\mathbf{2169.58 \pm 829.18}$ & $\mathbf{2262.65 \pm 811.30}$ & $\mathbf{2437.21 \pm 817.10}$ \\
        \bottomrule
    \end{tabular}
\end{table}

\paragraph{Additional environment.}
We also evaluate on Swimmer ($\epsilon = 0.15$). As shown in Table~\ref{tab:swimmer}, Advis achieves the highest worst-case return and outperforms Radial.

\begin{table}[htbp]
    \centering
    \small
    \caption{Worst-case returns on Swimmer.}
    \label{tab:swimmer}
    \begin{tabular}{lc}
        \toprule
        Method & Minimal return \\
        \midrule
        SA-Reg & $73.60 \pm 8.4$ \\
        Radial & $55.55 \pm 61.47$ \\
        Advis & $\mathbf{78.31} \pm 4.74$ \\
        \bottomrule
    \end{tabular}
\end{table}

\paragraph{Different attack budgets.}
Finally, we evaluate against RS attacks with different attack budgets. Table~\ref{tab:additional-budgets} shows that Advis consistently outperforms SA-Reg and Radial across the tested values of $\epsilon$.

\begin{table}[htbp]
    \centering
    \small
    \caption{Returns under RS attacks with different attack budgets.}
    \label{tab:additional-budgets}
    \begin{tabular}{llccc}
        \toprule
        Env. & Method & \multicolumn{3}{c}{Attack budget} \\
        \midrule
        \multirow{4}{*}{HalfCheetah}
        & & $\epsilon=0.05$ & $\epsilon=0.075$ & $\epsilon=0.15$ \\
        & SA-Reg & $3940.6 \pm 39$ & $3908.7 \pm 39.52$ & $3795.7 \pm 92$ \\
        & Radial & $4035.3 \pm 533$ & $4035.76 \pm 442$ & $4028.1 \pm 533$ \\
        & Advis & $\mathbf{4644.4 \pm 252}$ & $\mathbf{4574.8 \pm 356}$ & $\mathbf{4118.8 \pm 62}$ \\
        \midrule
        \multirow{4}{*}{Hopper}
        & & $\epsilon=0.05$ & $\epsilon=0.075$ & $\epsilon=0.1$ \\
        & SA-Reg & $1365.7 \pm 818$ & $1363.32 \pm 265$ & $1079.25 \pm 559$ \\
        & Radial & $1003.7 \pm 35$ & $858.6 \pm 52$ & $820.43 \pm 76$ \\
        & Advis & $\mathbf{2005.6 \pm 438}$ & $\mathbf{1613.1 \pm 307}$ & $\mathbf{1116.3 \pm 504}$ \\
        \midrule
        \multirow{4}{*}{Walker2d}
        & & $\epsilon=0.05$ & $\epsilon=0.075$ & $\epsilon=0.1$ \\
        & SA-Reg & $802.6 \pm 1691$ & $410.3 \pm 58$ & $107.1 \pm 55$ \\
        & Radial & $543.3 \pm 658$ & $542 \pm 303$ & $475.7 \pm 8$ \\
        & Advis & $\mathbf{3189.6 \pm 1222.9}$ & $\mathbf{1935.5 \pm 933}$ & $\mathbf{909.9 \pm 997}$ \\
        \bottomrule
    \end{tabular}
\end{table}

\end{document}